\documentclass{article}
\usepackage{microtype}
\usepackage{graphicx}
\usepackage{subcaption}
\usepackage{multirow}
\usepackage{multicol}
\usepackage{booktabs}
\usepackage{threeparttable}
\usepackage{hyperref}
\usepackage{adjustbox}
\usepackage{xcolor}

\usepackage[table]{xcolor}
\usepackage{makecell}

\usepackage[preprint]{icml2026}

\usepackage{amsmath}
\usepackage{amssymb}
\usepackage{mathtools}
\usepackage{amsthm}
\usepackage{bm}

\usepackage[capitalize,noabbrev]{cleveref}

\theoremstyle{plain}
\newtheorem{theorem}{Theorem}[section]
\newtheorem{proposition}[theorem]{Proposition}

\theoremstyle{definition}

\theoremstyle{remark}

\newcommand{\ie}{\textit{i.e.},~}
\newcommand{\eg}{\textit{e.g.},~}

\usepackage[textsize=tiny]{todonotes}

\icmltitlerunning{Generalized Evidential Deep Learning: From a Bayesian Perspective}

\begin{document}

\twocolumn[
  \icmltitle{Generalized Evidential Deep Learning: From a Bayesian Perspective}

  \icmlsetsymbol{equal}{*}

  \begin{icmlauthorlist}
    \icmlauthor{Yuanye Liu}{sch}
    \icmlauthor{Yibo Gao}{sch}
    \icmlauthor{Yuanyang Chen}{sch}
    \icmlauthor{Xiahai Zhuang}{sch}
  \end{icmlauthorlist}

  \icmlaffiliation{sch}{School of Data Science, Fudan University, Shanghai, China}

  \icmlcorrespondingauthor{Xiahai Zhuang}{zxh@fudan.edu.cn}

  \icmlkeywords{Machine Learning, ICML}

  \vskip 0.3in
]

\printAffiliationsAndNotice{}

\begin{abstract}
Evidential Deep Learning (EDL) has emerged as an efficient, sampling-free strategy for uncertainty estimation.
A series of EDL variants have been proposed to address specific limitations of the original framework, achieving notable success.
However, the underlying theoretical structure of EDL and the relationships among these variants have received limited systematic investigation.
In this work, we establish a principled theoretical foundation for EDL by interpreting it within a generalized Bayesian framework that includes prior specification, posterior update, and training objective.
We further characterize evidential uncertainty from a Bayesian distributional uncertainty viewpoint, established via asymptotic analysis.
Building on this perspective, we further propose Generalized Evidential Deep Learning (GEDL), a unified and extensible framework that explicitly disentangles the roles of individual components and systematically relates GEDL to existing variants.
Extensive experiments demonstrate that GEDL yields comparable results on classification, uncertainty estimation and OOD detections, with theoretical grounding.
\end{abstract}

\section{Introduction}

Uncertainty estimation lies at the core of reliable machine learning systems~\cite{J_2023AIR_uncertainty_survey}.
Evidential Deep Learning (EDL) offers a sampling-free approach to uncertainty estimation by predicting parameters of conjugate distributions, most commonly the Dirichlet distribution for classification~\cite{C_2018NIPS_EDL}.
By interpreting network outputs as evidence supporting each class, EDL produces closed-form predictive uncertainty and has demonstrated promising empirical performance across a variety of tasks~\cite{C_2023CoRL_EDL_drive,J_2025MIA_MERIT}.

\begin{figure}[t]
    \centering
    \includegraphics[height=6.38cm]{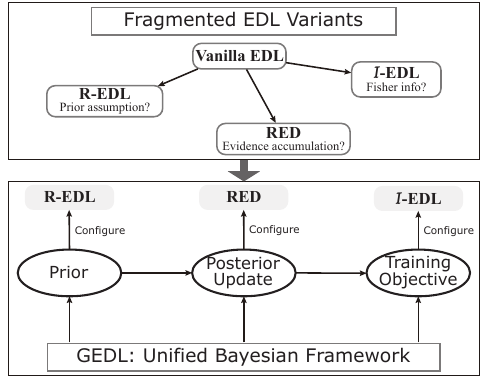}
    \caption{Overview of EDL variants and the proposed GEDL framework. Existing variants address different issues in a fragmented manner. GEDL unifies them under a single Bayesian framework by interpreting each variant as a principled choice, offering a coherent and theoretically grounded view of EDL.}
    \label{fig:teaser}
    \vspace{-1.3em}
\end{figure}
However, several studies found that the behavior of EDL models is highly sensitive to design choices in evidence modeling and regularization, leading to unstable performance and uncertainty estimates.
This has motivated a growing number of EDL variants that introduce targeted modifications to address specific empirical issues,
such as $\mathcal{I}$-EDL~\cite{C_2023ICLR_IEDL}, R-EDL~\cite{C_2024ICLR_REDL}, and RED~\cite{C_2023ICML_RED}.
While these methods achieve improvements in particular settings, they are often developed in isolation and lack a shared theoretical foundation.
Consequently, the EDL landscape has become increasingly fragmented, as shown in Fig.~\ref{fig:teaser}.
It is essential to make this framework explicit for understanding the behavior of EDL models, unifying existing variants, and designing principled improvements.
Therefore, we ask

\begin{it}
    What exactly is EDL doing from a Bayesian perspective?
\end{it}

To this end, we reformulate Evidential Deep Learning within a generalized Bayesian framework.
By explicitly mapping each component of Bayesian inference to its counterpart in EDL, we provide a unified interpretation of prior specification, posterior updating, and training objective in evidential models.
In particular, we show that the evidence predicted by neural networks can be understood as input-dependent pseudo-counts in a generalized Bayesian posterior, and that commonly used EDL losses correspond to variational objectives under tempered Bayesian updating.
From this perspective, evidential uncertainty admits a principled interpretation as Bayesian distributional uncertainty, supported by an asymptotic analysis.

Building on this unified view, as shown in Fig.~\ref{fig:teaser}, we further propose Generalized Evidential Deep Learning (GEDL), a modular Bayesian framework that makes each component of EDL explicit and configurable.
Within GEDL, existing EDL variants can be naturally interpreted from a Bayesian perspective, while also enabling systematic extensions beyond current designs.

Our contributions are summarized as follows:
\begin{itemize}
    \item We provide a principled Bayesian interpretation of EDL, establishing formal connections between Bayesian generative processes and key components of EDL, and further interpreting evidential uncertainty as Bayesian distributional uncertainty.
    \item We propose Generalized Evidential Deep Learning (GEDL), a unified and configurable framework that subsumes existing EDL variants within a single probabilistic model.
    \item Through extensive experiments, we show that GEDL achieves performance comparable to or better than existing EDL methods, while offering a more theoretically grounded and reliable approach to uncertainty estimation.
\end{itemize}

\section{Preliminary and Related Work}
\subsection{Bayesian Framework for Uncertainty Estimation}
Bayesian inference establishes a principled framework for uncertainty-aware reasoning by explicitly encoding uncertainty in probabilistic distributions~\cite{J_2023AIR_Bayes_review}.
At its core, the Bayesian framework is characterized by three essential components: prior beliefs that encode inductive assumptions, posterior updates that incorporate observed data, and uncertainty quantification through posterior distributions. 
This paradigm offers a coherent and interpretable foundation for uncertainty estimation, as uncertainty is treated as an intrinsic property of the inference process, rather than an auxiliary output.

In the context of predictive modeling, Bayesian methods naturally support uncertainty-aware decision making by propagating uncertainty from latent variables to predictions. 
Unlike point-estimate approaches, Bayesian models represent predictive uncertainty as a distribution, enabling richer characterizations of different uncertainty.

\subsection{Evidential Deep Learning}

\begin{figure*}[t]
  \centering
  \subcaptionbox{\label{fig:pgm}}{
    \includegraphics[height=5.26cm]{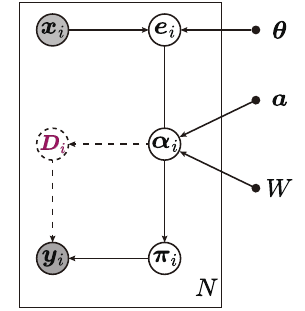}
  }
  \hfill
  \subcaptionbox{\label{fig:relationship}}{
    \includegraphics[height=5.26cm]{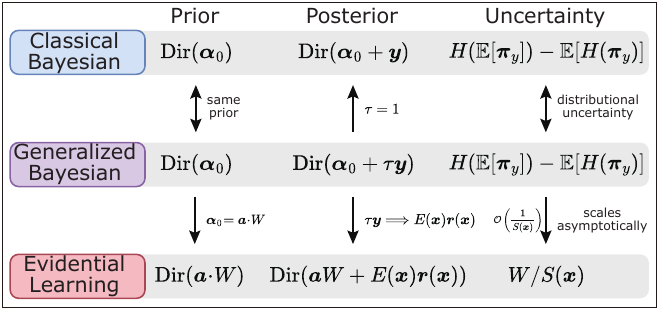}
  }
  \caption{(a) Probabilistic graphical model of EDL. 
(b) Connections between EDL and Bayesian inference from the perspectives of prior, posterior update, and uncertainty. EDL emerges as a parameterized generalized Bayesian framework whose uncertainty mass corresponds to Bayesian distributional uncertainty.
  }
  \label{fig:main}
\end{figure*}

In Evidential Deep Learning, predictions are represented as \emph{subjective opinions} that explicitly allocate belief mass and uncertainty mass, \ie $\bm{D}=\{b_1,\dots,b_K,u\}$, inferred by neural networks directly from data.
As illustrated in Fig.~\ref{fig:pgm}, given an input $\bm{x}$, a neural network parameterized by $\bm{\theta}$ outputs a non-negative evidence vector $\bm{e}(\bm{x})=f_{\bm{\theta}}(\bm{x})=\{e_1,\dots,e_K\}$.
This evidence is then mapped to the concentration parameters of a Dirichlet distribution over the latent class probability vector $\bm{\pi}$ via
\begin{equation}\label{eq:e2alpha}
    \bm{\alpha}(\bm x)=\bm{e}(\bm{x})+W\cdot\bm{a},
\end{equation}
where $\bm{a}$ denotes the \emph{base rate} encoding prior class preference, and $W$ is the \emph{prior strength}.
This construction induces a Dirichlet posterior over $\bm{\pi}$ and enables closed-form prediction and uncertainty estimation in a single forward pass.
The evidence vector $\bm{e}(\bm{x})$ is decomposed into class-wise belief masses and an uncertainty mass as
\begin{equation}\label{eq:compute_opinion}
b_k(\bm{x}) = \frac{e_k(\bm{x})}{S(\bm{x})}, 
\quad 
u(\bm{x}) = \frac{W}{S(\bm{x})}, 
\quad 
S(\bm{x})=\sum_{k}\alpha_k(\bm{x}),
\end{equation}
which together form a subjective opinion.
The resulting predictive probability admits an equivalent Bayesian interpretation as the posterior predictive mean of the Dirichlet distribution,
\begin{equation}\label{eq:opinion2prob}
p_k(\bm{x})
= b_k(\bm{x}) + u(\bm{x}) a_k
= \mathbb{E}_{\bm{\pi}\sim \mathrm{Dir}(\bm{\alpha}(\bm{x}))}[\pi_k]
= \frac{\alpha_k(\bm{x})}{S(\bm{x})},
\end{equation}
which naturally interpolates between data-driven belief and prior knowledge encoded by the base rate $\bm{a}$.

Subsequent work has identified several limitations of the vanilla EDL formulation~\cite{C_2018NIPS_EDL}, which motivated targeted extensions that modify specific components of the evidential learning pipeline.
R-EDL~\cite{C_2024ICLR_REDL} relaxes the restriction of a fixed prior strength $W$, 
to improve uncertainty calibration.
$\mathcal{I}$-EDL~\cite{C_2023ICLR_IEDL} adopts an information-theoretic perspective and regulates evidence accumulation through the Fisher Information Matrix, to prevent excessive confidence and stabilize uncertainty estimates.
RED~\cite{C_2023ICML_RED} introduces weighting strategies to avoid zero-evidence regions through correct evidence regularization.
These methods demonstrate that performance and uncertainty behavior in EDL are highly sensitive to different design choices in evidential modeling, including how evidence is generated, weighted, and regularized during training.
However, existing variants typically modify these components in isolation and are largely motivated by empirical observations.
This motivates the need for a unified and principled Bayesian framework to systematically interpret and design evidential 
learning methods.

\section{Evidential Deep Learning acts as a Bayesianist}
We interpret EDL as a parameterized form of generalized Bayesian inference, consisting of prior specification, posterior update, training objective, and uncertainty quantification.
In this view, as shown in Fig.~\ref{fig:relationship}, the process of the forward evidence-to-Dirichlet corresponds to a generalized posterior update rule, while commonly used EDL loss functions admit a variational inference interpretation with a tempered likelihood.

Furthermore, we analyze the uncertainty mass produced by EDL from the perspective of Bayesian distributional uncertainty, and show that it is consistent with established statistical measures in the asymptotic regime.

All notations used throughout this paper are summarized in Table~\ref{tab:notation} in Appendix~\ref{sec:notation}.

\subsection{Prior}
Under the Bayesian interpretation of EDL, the base rate $\bm{a}$ represents a prior belief over class probabilities.
As shown in Eq.~\eqref{eq:e2alpha}, when no data-dependent evidence is provided, the induced Dirichlet distribution is entirely determined by the base rate and its associated strength $W$.
This interpretation becomes explicit through the mapping from subjective opinions to predictive probabilities in Eq.~\eqref{eq:opinion2prob}.
In the absence of evidence, where $\bm{b}=\bm{0}$ and $u=1$, the predictive distribution collapses to the base rate $\bm{a}$, reflecting pure prior belief.
Most existing EDL-based methods adopt a uniform prior, \ie $\bm{a} = (\tfrac{1}{K},\ldots,\tfrac{1}{K})$, implicitly assuming no preference among classes.
MERIT~\cite{J_2025MIA_MERIT} explores a more flexible specification by setting the base rate according to the empirical label distribution, allowing prior beliefs to reflect class imbalance in the data.

Beyond the base rate, the scalar $W$ in Eq.~\eqref{eq:e2alpha} controls the strength of the prior Dirichlet distribution.
In the vanilla EDL formulation, $W$ is fixed to the number of classes $K$, ensuring that in the absence of evidence the prior reduces to a uniform Dirichlet with $\bm{\alpha}=0+W\cdot \bm{a}=\bm{1}$.

Recent work such as R-EDL~\cite{C_2024ICLR_REDL} relaxes this constraint and treats $W$ as a tunable hyperparameter, reporting empirical performance gains.
From a Bayesian perspective, this is a natural extension: the prior strength should encode how strongly the model trusts its prior belief before observing evidence.
Consistent with this view, prior study~\cite{C_2022_prior_weight} suggests setting $W$ to a small constant, \eg $W=2$, to represent a weakly informative prior in presence of evidence.

\subsection{Posterior Update}\label{subsec:post_update}

Posterior updating describes how prior beliefs are revised given new observations~\cite{B_1995_bayesian}. 
For multi-class classification with categorical observations, the Dirichlet distribution is the conjugate prior of the categorical likelihood, leading to closed-form posterior updates~\cite{B_2006_PRML}.

Let $\bm{\pi}=(\pi_1,\dots,\pi_K)\in\Delta^{K-1}$ denote the latent class probability vector and assume a Dirichlet prior
$
\bm{\pi} \sim \mathrm{Dir}(\bm{\alpha}_0),
$
where $\bm{\alpha}_0=(\alpha_{0,1},\dots,\alpha_{0,K})$ encodes prior beliefs. 
Given an observed label $y$, represented as a one-hot vector $\bm{y}$, the categorical likelihood is
$
p(\bm{y}\mid\bm{\pi})=\prod_{k=1}^K \pi_k^{y_k}.
$
By conjugacy, the posterior distribution admits a closed-form update,
\begin{equation}
p(\bm{\pi}\mid\bm{y})=\mathrm{Dir}(\bm{\alpha}_0+\bm{y}),
\end{equation}
where the posterior concentration parameters are obtained by adding class-wise counts to the prior. 
More generally, for multiple observations with count vector $\bm{n}$, the update takes the additive form
$
\bm{\alpha}=\bm{\alpha}_0+\bm{n}.
$
This reveals posterior parameters are fully determined by the prior concentration and the sufficient statistics of the data, in Dirichlet–categorical conjugacy. Detailed derivation can be found in Appendix.\ref{prop:conj_dir}.

\noindent\textbf{Generalized posterior update.}
Classical Bayesian updating assumes that each observation contributes equally to the posterior. 
Generalized Bayesian inference relaxes this assumption by introducing a temperature parameter $\tau>0$ to control the influence of the likelihood~\cite{J_2001_bayesian_consistency,J_2016_General_Bayes}. 
For a Dirichlet prior and categorical likelihood, the tempered posterior remains in the Dirichlet family:
\begin{equation}\label{eq:scaled_posterior}
q_\tau(\bm{\pi}\mid\bm{y})
\;\propto\;
p(\bm{y}\mid\bm{\pi})^{\tau} p(\bm{\pi})
\;=\;
\mathrm{Dir}\big(\bm{\alpha}_0 + \tau \bm{y}\big).
\end{equation}

Compared with the classical update, the generalized posterior replaces the unit count $\bm{y}$ with a \emph{scaled pseudo-count} $\tau\bm{y}$. 
The temperature $\tau$ thus admits a clear Bayesian interpretation as an \emph{evidence strength} parameter, controlling the relative influence of the observation with respect to the prior.
Detailed derivation can be found in Appendix.\ref{prop:gen_conj_dir}.

\paragraph{Posterior Update in EDL.}
In EDL, the predictive distribution is parameterized as a Dirichlet distribution whose concentration parameters are obtained via an evidence-to-$\bm{\alpha}$ transformation in Eq.~\eqref{eq:e2alpha}.
Here, we show this transformation admits a direct Bayesian interpretation as a parameterized generalized posterior update. 

EDL can be viewed as learning the sufficient statistics of Dirichlet distribution in a data-dependent manner.
Specifically, Eq.~\eqref{eq:e2alpha} corresponds to replacing the discrete count increment $\tau\bm{y}$ in Eq.~\eqref{eq:scaled_posterior} with a continuous, input-conditioned pseudo-count vector $\bm{e}(x)$, thereby defining an implicit generalized posterior
\begin{equation}
q_{\bm{\theta}}(\bm{\pi}\mid \bm{x})
=
\mathrm{Dir}\big(\bm{\pi}\mid W\bm{a}+\bm{e}(\bm{x};\bm{\theta})\big),
\label{eq:edl_posterior}
\end{equation}
Under this interpretation, the network parameterizes a posterior update operator that maps an input $\bm{x}$ to a Dirichlet posterior over the latent categorical probabilities $\bm{\pi}\in\Delta^{K-1}$, instead of simply point estimation.

The evidence vector $\bm{e}(\bm{x})$ therefore inherits a precise Bayesian semantics: it acts as a vector of pseudo-counts that controls the posterior concentration.
To make this interpretation explicit, it is convenient to decompose the evidence into a \emph{strength} component and a \emph{direction} component,
\begin{equation}
    \bm{e}(\bm{x})=\|\bm{e}(\bm{x})\|\cdot\frac{\bm{e}(\bm{x})}{\|\bm{e}(\bm{x})\|}\triangleq E(\bm{x})\bm{r}(\bm{x}),
\end{equation}
where $E(\bm{x})\ge 0$ denotes the evidence strength and $\bm{r}(\bm{x})\in\Delta^{K-1}$ specifies the allocation of evidence across classes.
Under this decomposition, the posterior update in Eq.~\eqref{eq:e2alpha} and Eq.~\eqref{eq:edl_posterior} can be equivalently written as
\begin{equation}
\bm{\alpha}(\bm{x})
=
W\bm{a} + E(\bm{x})\,\bm{r}(\bm{x})
\;\approx\;
\bm{\alpha}_0 + \tau \bm{y},
\label{eq:strength_allocation}
\end{equation}
which makes the connection to tempered Bayesian updating explicit.
In the classical case, the one-hot label $\bm{y}$ specifies an extreme allocation of pseudo-counts, while the tempering parameter $\tau$ controls the magnitude of the update.
EDL generalizes this mechanism by allowing both the allocation $\bm{r}(\bm{x})$ and the strength $E(\bm{x})$ to be input-dependent and learned from data.

This interpretation also yields a principled view of the RED variant~\cite{C_2023ICML_RED}.
RED introduces an additional regularization that encourages higher evidence for correctly predicted samples, preventing them from collapsing into \emph{zero-evidence regions}.
From a generalized Bayesian perspective, this mechanism can be understood as modifying the effective pseudo-counts contributed by each observation, by enforcing a strengthened posterior update for correctly classified samples, \ie enlarging the term $\tau \bm{y}$.

\subsection{Training Objective from Variational Inference}\label{subsec:loss}

We now derive the training objective of EDL from a variational inference perspective.
This derivation clarifies the origin of commonly used EDL loss functions and provides a principled interpretation of the weighting applied to the regularization term~\cite{C_2018NIPS_EDL}.

We consider the latent categorical vector $\bm{\pi}$ and adopt the Dirichlet family as the variational distribution, following Eq.~\eqref{eq:edl_posterior}.
Our goal is to approximate the tempered posterior
$q_{\bm{\tau}}(\bm{\pi} \mid \bm{y})$ introduced in Eq.~\eqref{eq:scaled_posterior}
with a variational posterior $q_{\bm{\theta}}(\bm{\pi} \mid \bm{x})$ parameterized by the network output.
This approximation is obtained by minimizing the Kullback–Leibler divergence between the two distributions, which induces a variational optimization objective that admits the following decomposition:
\begin{equation}\label{eq:loss}
\begin{split}
    \mathrm{KL}&\!\left(
        q_{\theta}(\bm{\pi}\mid \bm{x})
        \;\|\;
        q_\tau(\bm{\pi} \mid \bm{y})
    \right)
    =\mathrm{KL}\!\left(q_{\theta}(\bm{\pi}\mid \bm{x})\;\|\;p(\bm{\pi})\right)\\
    &+ \tau\,\mathbb{E}_{q_{\bm{\theta}}(\bm{\pi}\mid \bm{x})}\!\big[-\log p(\bm{y}\mid \bm{\pi})\big] + \log(Z) ,
\end{split}
\end{equation}
where $Z$ is a normalization constant independent of $\theta$ (see Appendix~\ref{prop:elbo} for full derivation).
Then we could yield an equivalent objective that corresponds to minimizing the negative evidence lower bound (ELBO):
\begin{equation}
\label{eq:elbo}
\mathcal{L}_{\text{VI}}
=
\mathbb{E}_{q_{\bm{\theta}}(\bm{\pi}\mid \bm{x})}
\big[-\log p(\bm{y}\mid \bm{\pi})\big]
+
\frac{1}{\tau}\,
\mathrm{KL}\!\left(
q_{\bm{\theta}}(\bm{\pi}\mid \bm{x})
\;\|\;
p(\bm{\pi})
\right).
\end{equation}
The first term encourages accurate prediction under the posterior predictive distribution,
while the second term regularizes the variational posterior toward the prior.
The temperature parameter $\tau$ controls the trade-off between data fitting and prior regularization, which corresponds to the inverse of the KL weight commonly denoted as $\lambda$ in other EDL works.
Under the Dirichlet--categorical model, Eq.~\eqref{eq:elbo} admits closed-form expressions, given in Appendix~\ref{prop:closed-form}.

More fundamentally, $\mathcal{L}_{\mathrm{VI}}$ can be viewed as a special case of the generalized Bayesian updating framework~\cite{J_2016_General_Bayes}.
As shown in Appendix~\ref{prop:generalized_belief}, our objective is exactly equivalent to the Gibbs posterior induced by choosing the negative log-likelihood as the loss function.

Within the same generalized Bayesian framework, alternative loss functions can be naturally incorporated.
In particular, the MSE-based objective proposed in~\cite{C_2018NIPS_EDL} arises as a Gibbs posterior when the squared error replaces the log-likelihood in the variational objective (see Appendix~\ref{prop:MSE-loss}).
From this viewpoint, the loss function of $\mathcal{I}$-EDL~\cite{C_2023ICLR_IEDL} can also be interpreted as a Gibbs variational objective augmented with a Fisher-information regularizer, which explicitly constrains evidence accumulation and stabilizes uncertainty estimation.

Moreover, $\mathcal{L}_{\mathrm{VI}}$ admits an alternative theoretical interpretation from a PAC-Bayes perspective.
It aligns with PAC-Bayes-inspired learning objectives that balance empirical risk and posterior regularization~\cite{C_2021_bayes_pac,C_2023ICLR_IEDL}.
The derivation can be found in Appendix~\ref{prop:PACbound}.

Together, these interpretations demonstrate that $\mathcal{L}_{\mathrm{VI}}$ is not an ad-hoc design, but a unifying objective that connects variational inference, generalized Bayesian learning, and PAC-Bayes theory within a single coherent framework.

\begin{table*}[t]
\centering
\caption{Unified Bayesian interpretation of EDL variants. Each method is characterized in terms of prior strength, evidence strength, and training objective.}
\label{tab:unified_edl}
\resizebox{\textwidth}{!}
{
\begin{tabular}{lcccc}
\toprule
& \multicolumn{1}{c}{\textbf{Prior Strength $W$}} 
& \multicolumn{1}{c}{\textbf{Evidence Strength $\tau$}} 
& \multicolumn{1}{c}{\textbf{Training Objective}} \\
\midrule
Vanilla EDL
& Fixed as $K$ 
& $\tau_t^{-1}=\min(1,t/10)$ 
& KL only on misclassified samples  \\

$\mathcal{I}$-EDL
& Fixed as $K$ 
& $\tau_t^{-1}=\min(1,t/10)$
& KL only on misclassified samples with MSE-based likelihood \\

R-EDL
& Tuned hyperparameter 
& $\tau_t^{-1}=\min(1,t/10)$ 
& KL only on misclassified samples \\

RED
& Fixed as $K$ 
& $\tau_t^{-1}=
\begin{cases}
\min(1,t/10), & \text{incorrect} \\
u=W/S, & \text{correct}
\end{cases}$ 
& \makecell[c]{Regularize misclassified samples \\ 
Encourage higher evidence for correctly classified samples} \\
\midrule
\textbf{GEDL} 
& 
$\displaystyle W=\frac{K + C_w K \sum_k e_k(x)}{1 + K \sum_k e_k(x)}$ 

& 
$\displaystyle \tau_t=\frac{c}{\sum_{i\le t}\mathbb{E}[S(x_i)]}$ 

& Unified KL derived from ELBO with tempered likelihood\\
\bottomrule
\end{tabular}
}
\end{table*}

\subsection{Evidential Uncertainty from Bayesian Perspective}
Building on the established Bayesian interpretation of EDL,
we turn to the question of how evidential uncertainty ($u$ in Eq.~\eqref{eq:compute_opinion}) should be interpreted within this framework.

In a fully Bayesian treatment, the predictive distribution is obtained by marginalizing over both the latent class probability vector $\bm{\pi}$ and the model parameters $\bm{\theta}$,
\begin{equation}
        P(\bm{y}\mid\bm{x},\mathcal{D}) = \iint p(\bm{y}\mid\bm{\pi})\,p(\bm{\pi}\mid\bm{x},\theta)\,p(\theta\mid\mathcal{D})\,{\rm d}\bm{\pi} {\rm d}\theta.
\end{equation}
This predictive uncertainty can be decomposed into three conceptually distinct sources: aleatoric uncertainty arising from $p(\bm{y}\mid\bm{\pi})$, distributional uncertainty arising from $p(\bm{\pi}\mid\bm{x},\bm{\theta})$, and epistemic uncertainty arising from $p(\bm{\theta}\mid\mathcal{D})$.

EDL differs from fully Bayesian neural networks in that it replaces the posterior distribution over network parameters with a point estimate $\hat{\bm{\theta}}$.
This approximation can be viewed as substituting $p(\bm{\theta}\mid\mathcal{D})$ with a Dirac delta distribution $\delta(\bm{\theta}-\hat{\bm{\theta}})$.
Under this assumption, the Bayesian predictive distribution simplifies to
\begin{equation}
p(\bm{y}\mid\bm{x},\mathcal{D})
\approx \int p(\bm{y}\mid\bm{\pi})\,p(\bm{\pi}\mid\bm{x}\,\hat{\bm{\theta}})\, \mathrm{d}\bm{\pi},
\end{equation}
where the uncertainty associated with model parameters is no longer explicitly represented.
Consequently, all predictive uncertainty in EDL is encoded solely in the Dirichlet posterior over the latent class probability vector $\bm{\pi}$.
We propose that evidential uncertainty $u$ reflects \emph{distributional uncertainty} in our full Bayesian sense.

Following~\cite{C_2018NIPS_prior_net}, mutual information is used as a principled measure of Bayesian distributional uncertainty and is defined as
\begin{equation}\label{eq:dist_uncertainty}
    I(\bm{y}; \bm{\pi} \mid \bm{x}, \hat{\theta}) 
    = H\big(\mathbb E[\bm{\pi_{\bm{y}}}]\big) 
    - \mathbb{E} \big[ H(\bm{\pi}_{\bm{y}}) \big],
\end{equation}
where $\bm{\pi}_{\bm{y}}$ is short for $p(\bm{y}\mid \bm{\pi})$. 
As shown in Appendix.~\ref{prop:u}, when $\bm{\alpha}(\bm{x}) / S(\bm{x}) \to \bar{\bm{\pi}} \in \Delta^{K-1}$ as $S(\bm{x}) \to \infty$, we have the asymptotic expansion,
\begin{equation}\label{eq:dist_MI}
    I(\bm{y}; \bm{\pi} \mid \bm{x}, \hat{\theta}) 
    = \frac{K - 1}{2S(\bm{x})} + \mathcal{O}\!\left( \frac{1}{S(\bm{x})^2} \right),
\end{equation}
demonstrating that this mutual information form of distributional uncertainty is asymptotically proportional to $1 / S(\bm{x})$. 
Similarly, the sum of marginal variance of each component satisfies as well,
\begin{equation}\label{eq:dist_var_sum}
    \sum_k\mathrm{Var}[\pi_k \mid \bm{x}, \hat{\theta}] 
    = \frac{\sum_k\bar{\pi}_k (1 - \bar{\pi}_k)}{S(\bm{x}) + 1}
    = \mathcal{O}\!\left( \frac{1}{S(\bm{x})} \right),
\end{equation}
confirming that total dispersion in $\bm{\pi}$ scales with $1 / S(\bm{x})$. The derivation can be found in Appendix~\ref{prop:expectation}.

This asymptotic behavior of $u= W / {S(\bm{x})}$ aligns with recent critiques \cite{C_2022NIPS_pitfallEUQ, C_2024NIPS_areEDLmirage}, showing that EDL methods do not reduce epistemic uncertainty to zero, even as data size tend to be infinity.

\begin{table*}[t]
\centering

\caption{Quantitative comparison of uncertainty estimation methods on MNIST (top) and CIFAR-10 (bottom) for in-distribution classification and out-of-distribution detection. Results are reported as mean $\pm$ std over multiple runs. \textbf{Bold} indicates the best performance and \underline{underlined} indicates the second best. GEDL consistently achieves superior or comparable performance across both ID and OOD settings, demonstrating more reliable uncertainty modeling than existing EDL variants.}
\label{tab:results}

\resizebox{0.95\textwidth}{!}
{
\begin{threeparttable}

\begin{subtable}{\textwidth}
\centering
\begin{tabular*}{\textwidth}{@{\extracolsep{\fill}} l c c cc cc}
\toprule
\multirow{2}{*}{Method}
& \multicolumn{2}{c}{MNIST ID classification}
& \multicolumn{2}{c}{MNIST$\rightarrow$KMNIST}
& \multicolumn{2}{c}{MNIST$\rightarrow$FMNIST} \\
\cmidrule(lr){2-3}\cmidrule(lr){4-5}\cmidrule(lr){6-7}
& Acc  & Conf.MP & MP & UM & MP & UM \\
\midrule
MC Dropout & 99.26$\pm$0.00  & 99.98$\pm$0.00 & 94.00$\pm$0.10 & -- & 96.56$\pm$0.30 & -- \\
DUQ        & 98.65$\pm$0.10  & 99.97$\pm$0.00 & 98.52$\pm$0.10 & -- & 97.92$\pm$0.60 & -- \\
KL-PN      & 99.01$\pm$0.00  & 99.92$\pm$0.00 & 92.97$\pm$1.20 & 93.39$\pm$1.00 & 98.44$\pm$0.10 & 98.16$\pm$0.00 \\
RKL-PN     & 99.21$\pm$0.00  & 99.67$\pm$0.00 & 60.76$\pm$2.90 & 53.76$\pm$3.40 & 78.45$\pm$3.10 & 72.18$\pm$3.60 \\
PostN      & \underline{99.34$\pm$0.00} & 99.98$\pm$0.00 & 95.75$\pm$0.20 & 94.59$\pm$0.30 & 97.78$\pm$0.20 & 97.24$\pm$0.30 \\
\midrule
EDL        & 98.22$\pm$0.31 & \textbf{99.99$\pm$0.00} & 97.02$\pm$0.76 & 96.31$\pm$2.03 & 98.11$\pm$0.44 & 98.08$\pm$0.42 \\
$\mathcal{I}$-EDL & 99.21$\pm$0.08  & 99.98$\pm$0.00 & 98.34$\pm$0.24 & 98.33$\pm$0.24 & \underline{98.89$\pm$0.28} & 98.86$\pm$0.29 \\
R-EDL      & 99.33$\pm$0.03  & \textbf{99.99$\pm$0.00} & \underline{98.69$\pm$0.19} & \underline{98.69$\pm$0.20} & \textbf{99.29$\pm$0.11} & \textbf{99.29$\pm$0.12} \\
\midrule
\textbf{GEDL (Ours)} & \textbf{99.58$\pm$0.02}  & 99.95$\pm$0.01 & \textbf{99.80$\pm$0.01} & \textbf{99.70$\pm$0.04} & 98.54$\pm$0.13 & \underline{98.99$\pm$0.05} \\
\bottomrule
\end{tabular*}
\end{subtable}

\vspace{6pt}

\begin{subtable}{\textwidth}
\centering
\begin{tabular*}{\textwidth}{@{\extracolsep{\fill}} l c c cc cc}
\toprule
\multirow{2}{*}{Method}
& \multicolumn{2}{c}{CIFAR10 ID classification}
& \multicolumn{2}{c}{CIFAR10$\rightarrow$SVHN}
& \multicolumn{2}{c}{CIFAR10$\rightarrow$CIFAR100} \\
\cmidrule(lr){2-3}\cmidrule(lr){4-5}\cmidrule(lr){6-7}
& Acc & Conf.MP & MP & UM & MP & UM \\

\midrule
MC Dropout & 82.84$\pm$0.10  & 97.15$\pm$0.00 & 51.39$\pm$0.10 & -- & 45.57$\pm$1.00 & -- \\
DUQ        & 89.33$\pm$0.20  & 97.89$\pm$0.30 & 80.23$\pm$3.40 & -- & 84.75$\pm$1.10 & -- \\
KL-PN      & 27.46$\pm$1.70  & 50.61$\pm$4.00 & 43.96$\pm$1.90 & 43.23$\pm$2.30 & 61.41$\pm$2.8 & 61.53$\pm$3.40 \\
RKL-PN     & 64.76$\pm$0.30  & 86.11$\pm$0.40 & 53.61$\pm$1.10 & 49.37$\pm$0.80 & 55.42$\pm$2.6 & 54.74$\pm$2.80 \\
PostN      & 84.85$\pm$0.00 & 97.76$\pm$0.00 & 80.21$\pm$0.20 & 77.71$\pm$0.30 & 81.96$\pm$0.8 & 82.06$\pm$0.80 \\
\midrule
EDL        & 83.55$\pm$0.64 & 97.86$\pm$0.17 & 78.87$\pm$3.50 & 79.12$\pm$3.69 & 84.30$\pm$0.67 & 84.18$\pm$0.74 \\
$\mathcal{I}$-EDL      & 89.20$\pm$0.32 & \underline{98.72$\pm$0.12} & 83.26$\pm$2.44 & 82.96$\pm$2.17 & 85.35$\pm$0.69 & 84.84$\pm$0.64 \\
R-EDL      & \underline{90.09$\pm$0.30}& \textbf{98.98$\pm$0.05} & \underline{85.00$\pm$1.22} & \underline{85.00$\pm$1.22} & \underline{87.72$\pm$0.31} & \textbf{87.73$\pm$0.31} \\
\addlinespace[2pt]
\midrule
\textbf{GEDL (Ours)} & \textbf{90.22$\pm$0.49}  & 97.15$\pm$3.72 & \textbf{89.26$\pm$3.26} & \textbf{92.56$\pm$6.02} & \textbf{88.27$\pm$2.51} & \underline{85.42$\pm$5.14} \\
\bottomrule
\end{tabular*}
\end{subtable}

\end{threeparttable}
}
\end{table*}

\section{Generalized Evidential Deep Learning}

Building on the Bayesian interpretation of evidential learning, we now address a more fundamental question:

\emph{What constitutes a complete and principled evidential inference framework from a Bayesian perspective?}

We answer this question by proposing \emph{Generalized Evidential Deep Learning (GEDL)}, a unified formulation that explicitly characterizes the three core components of evidential learning under a Bayesian framework: prior specification, posterior update, and training objective.
By disentangling these components, GEDL provides a transparent view of how different design choices influence posterior behavior and uncertainty estimation, and offers a principled foundation for systematically extending existing EDL methods.

A comparative summary of GEDL and representative EDL variants is presented in Table~\ref{tab:unified_edl}.

\paragraph{Prior specification.}

In GEDL, we adopt a convergent prior scheme for the prior strength, inspired by~\cite{C_2022_prior_weight}.
Rather than fixing the prior strength $W$ as a constant, we allow it to vary smoothly as a function of the accumulated evidence.

As formalized in Eq.~\eqref{eq:weight_update}, the prior strength is a bounded, monotonic function of the total evidence $\sum_k e(x)$, and converges to a predefined constant $C_w$ as evidence increases.
This design ensures that the prior remains influential in low-evidence regimes, while its relative impact diminishes as sufficient evidence is accumulated, preventing over-regularization in confident predictions.
\begin{equation}
W = \frac{K + C_w\, K \sum_k e_k(x)}{1 + K \sum_k e_k(x)} .
\label{eq:weight_update}
\end{equation}

\paragraph{Posterior update.}

As discussed in Sec.\ref{subsec:post_update}, the posterior update in EDL can be interpreted as a generalized Bayesian update tempered by an evidence strength parameter $\tau$.
Most existing EDL methods rely on heuristic scheduling strategies for this parameter. A common practice is to anneal the KL coefficient during training, for example, using $\tau^{-1}=\lambda=\min(1, t/10)$\cite{C_2018NIPS_EDL}.
Although empirically effective, such schedules lack a clear Bayesian justification.

Under our Bayesian interpretation, the KL weight directly corresponds to the inverse of the tempering parameter $\tau$ and therefore governs the strength of evidence accumulation in generalized Bayesian inference.
Specifically, we employ a scheduled evidence strength $\tau$ that evolves according to
\begin{equation}
\tau_t
= \frac{C_{\tau}}{\sum_{i\le t}\mathbb E[S(x_i)]} ,
\label{eq:tau_schedule}
\end{equation}
where $t$ denotes the index of epoch, $S(\bm{x}_i)$ denotes the Dirichlet strength at training step $i$ and $C_{\tau}$ is a constant.
We interpret the annealing of $\tau$ as controlling the effective sample size of generalized Bayesian updating, which is consistent with~\cite{J_2001_bayesian_consistency}.
The detailed implementation and sensitivity analysis for $\tau_t$ can be found in Appendix.~\ref{apdx:exp_set} and~\ref{apdx:addition_results}.

\begin{figure*}[t]
    \begin{tabular}{cccc}
         \includegraphics[width=0.23\textwidth]{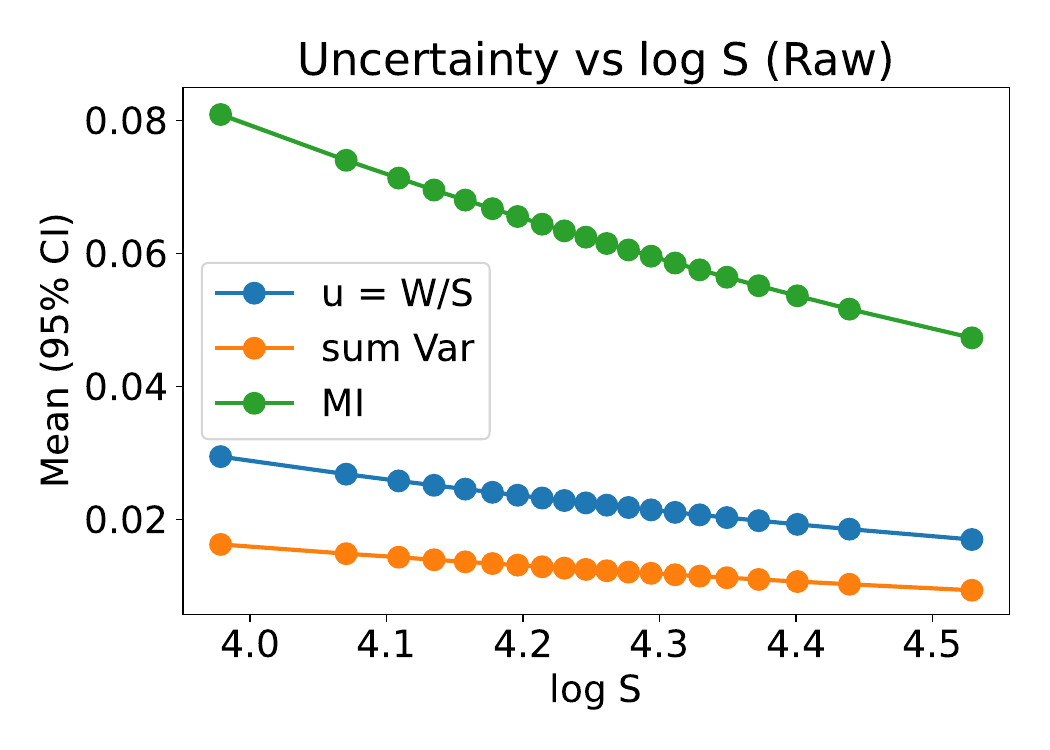}
     &
     \includegraphics[width=0.23\textwidth]{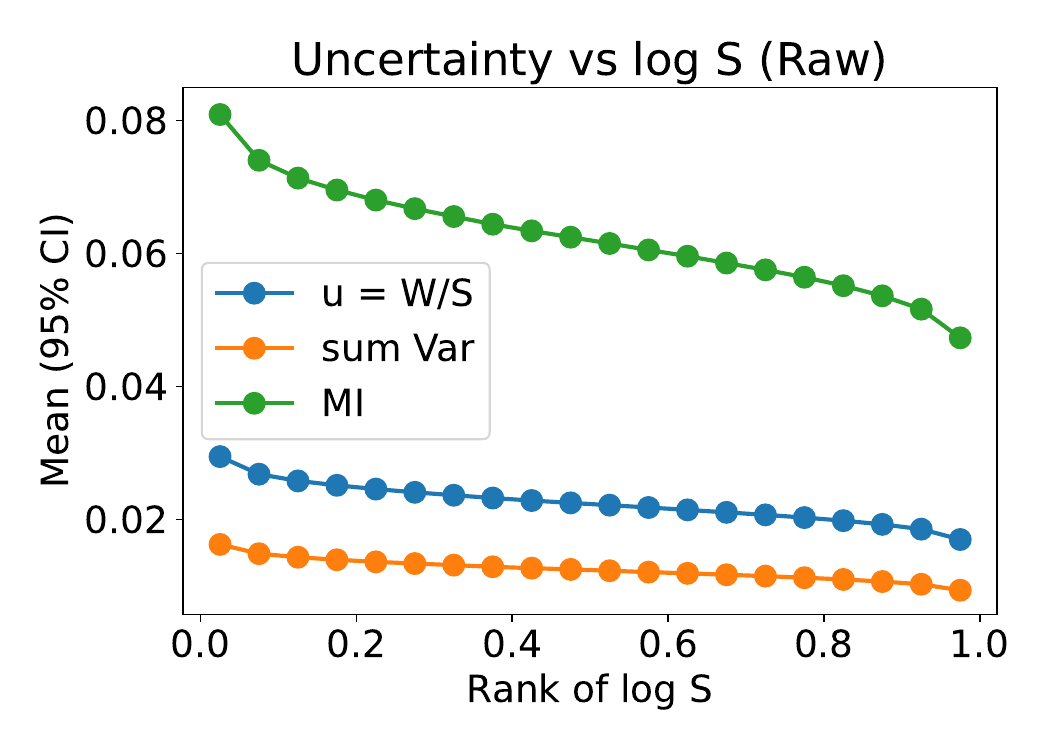}
    &
     \includegraphics[width=0.23\textwidth]{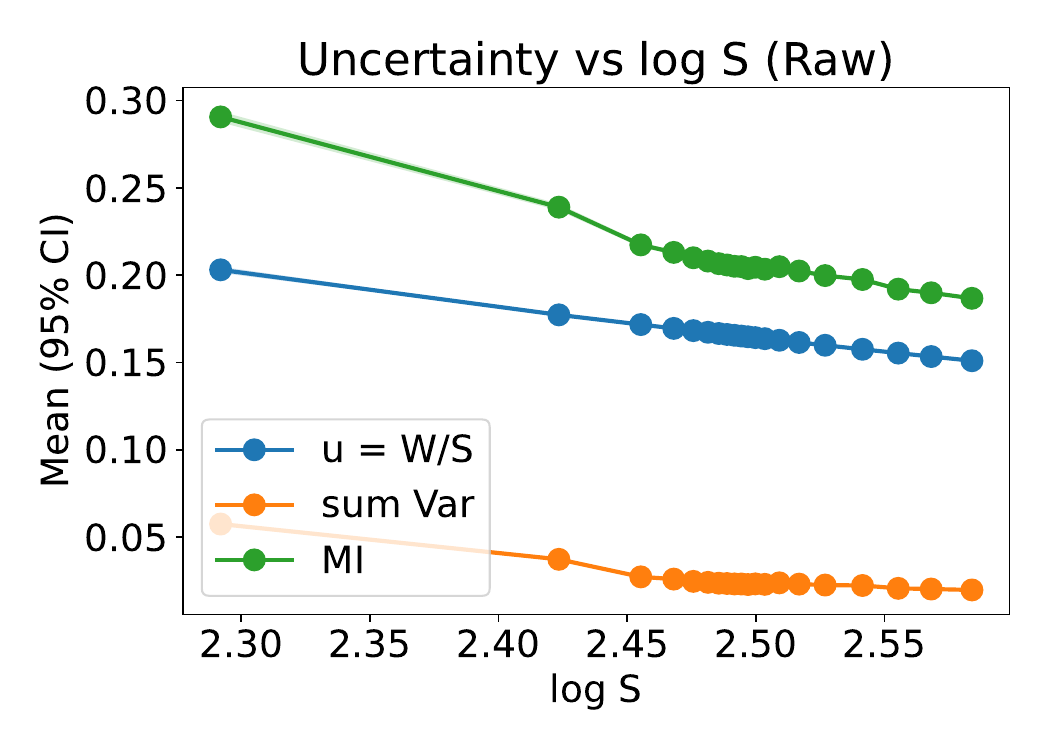}
     &
     \includegraphics[width=0.23\textwidth]{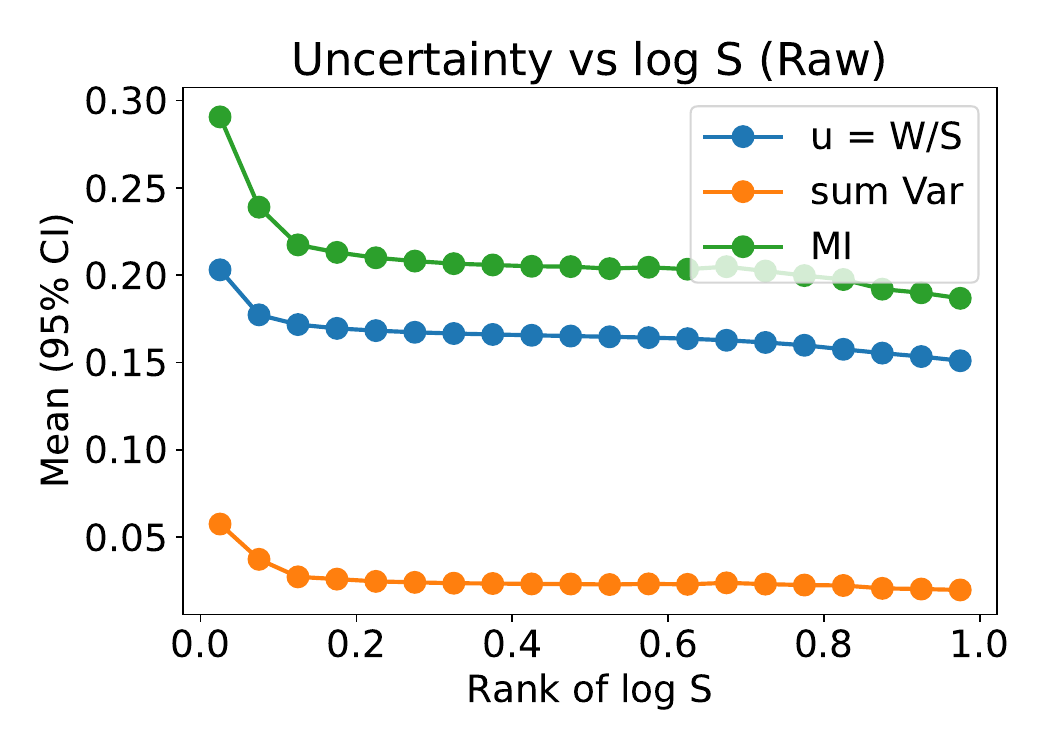}
     \\
     \includegraphics[width=0.23\textwidth]{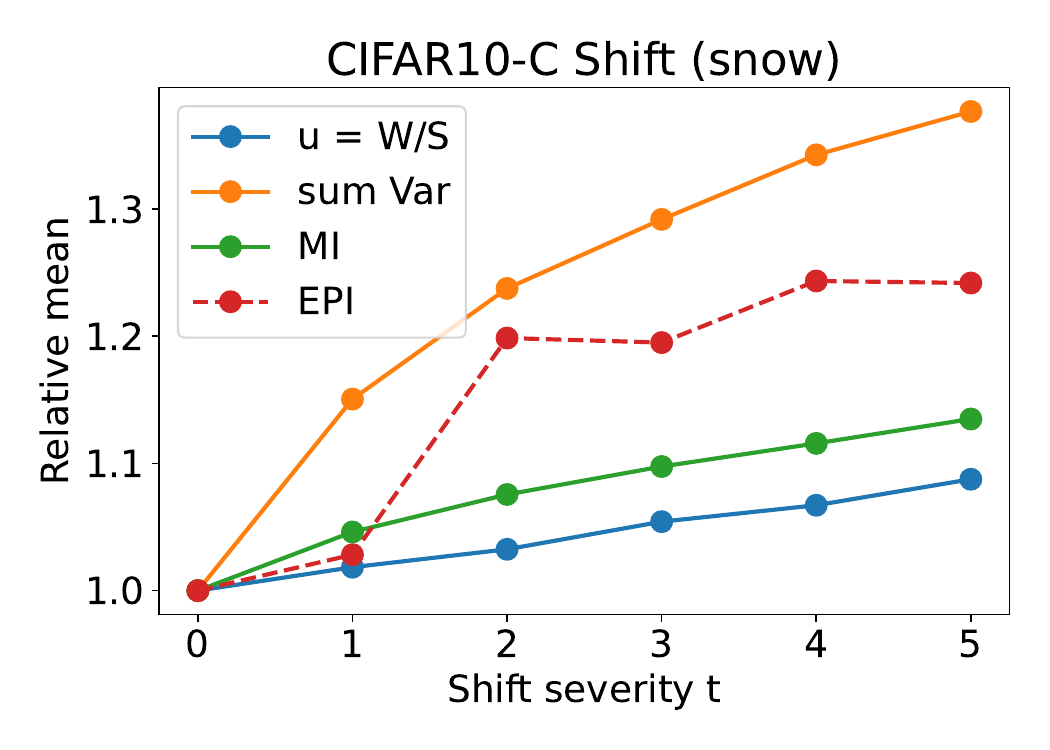}
     &
     \includegraphics[width=0.23\textwidth]{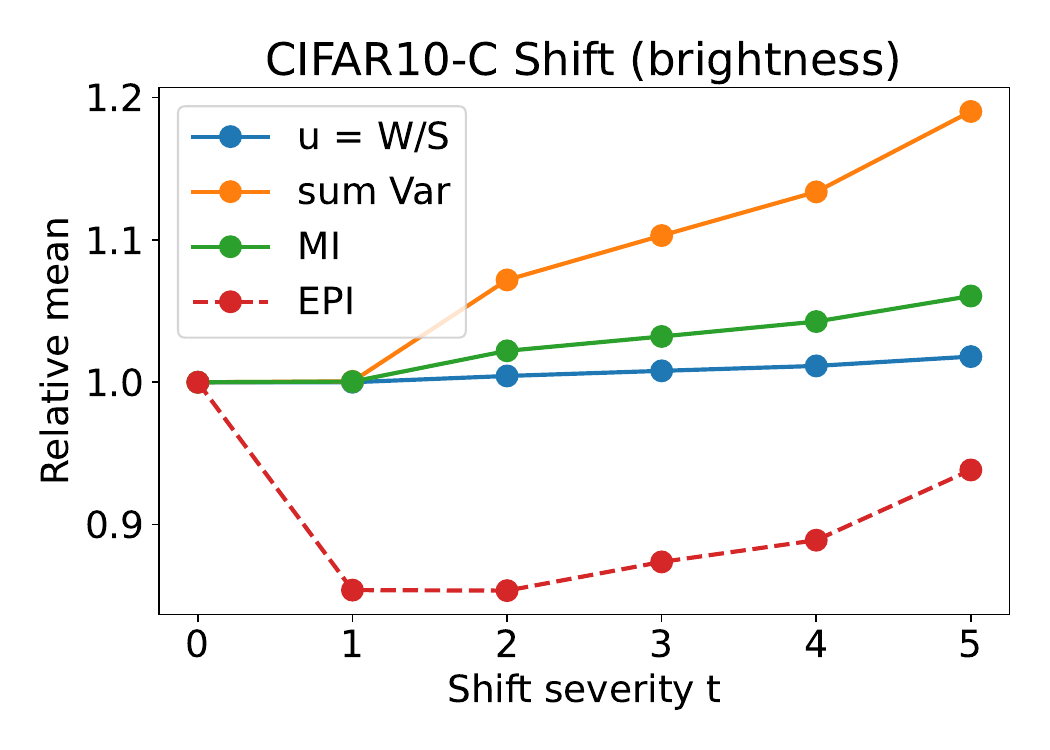}
    &
     \includegraphics[width=0.23\textwidth]{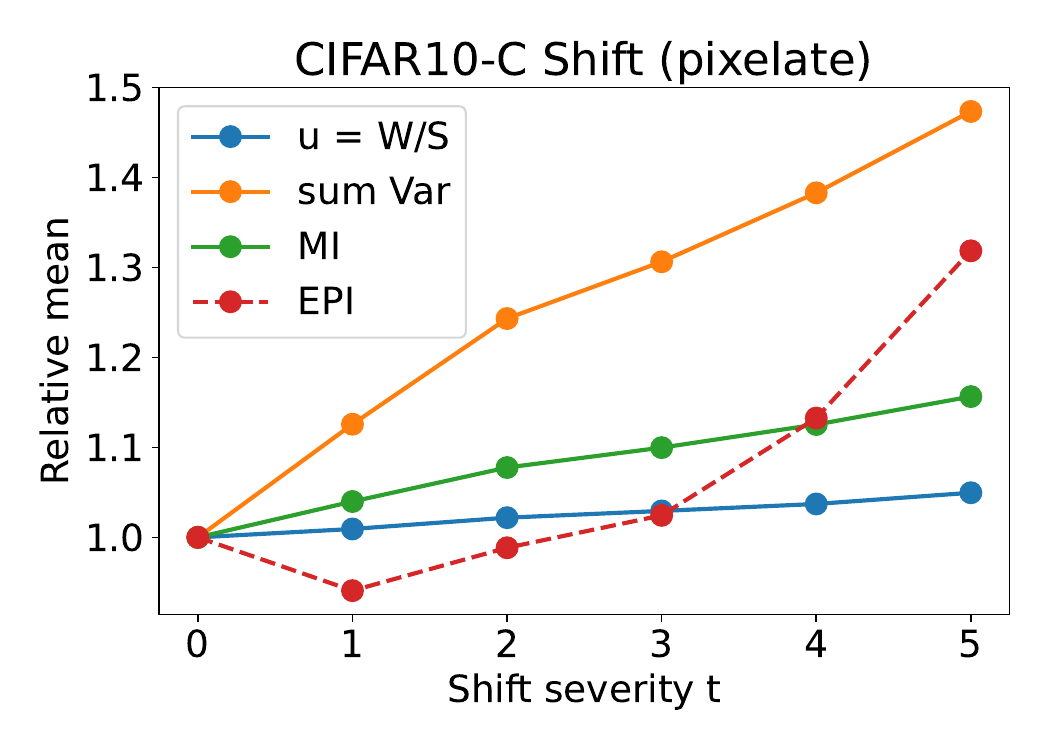}
     &
     \includegraphics[width=0.23\textwidth]{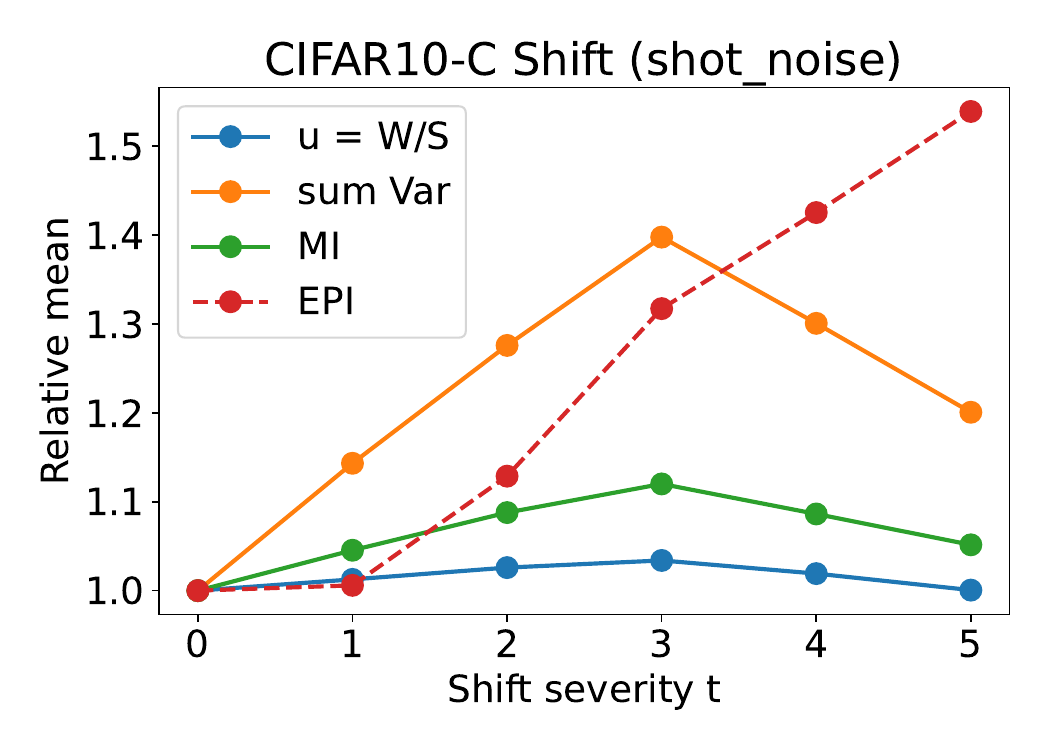}   
     \end{tabular}
     \caption{Empirical comparison of evidential uncertainty with Bayesian distributional uncertainty measures on clean data (top row) and corrupted data from CIFAR-10-C (bottom row). 
     Evidential uncertainty $u = W/S$ closely follows mutual information and predictive variance as $\log S$ or corruption severity increases, while ensemble epistemic uncertainty exhibits distinct trends.}
     \label{fig:uncertainty}
\end{figure*}
\vspace{-1em}

\paragraph{Training objective.}

As established in~Sec.\ref{subsec:loss}, the training objective of evidential learning admits a principled interpretation from a variational inference perspective with a form in Eq.~\eqref{eq:elbo}.
Most works modify the KL term by selectively penalizing only incorrectly classified samples, while preserving the evidence of correctly classified samples, \eg$\hat{\bm{\alpha}}_i = \bm{y}_i + (1-\bm{y}_i)\odot\bm{\alpha}_i$.
These designs are motivated by practical considerations, but break the consistency of posterior regularization and obscure the probabilistic meaning of the resulting objective.

More recently, \cite{C_2023ICML_RED} introduced an additional regularization term that explicitly promotes higher evidence for correctly classified samples, using an uncertainty-dependent coefficient $\lambda_{\mathrm{cor}} = u = W / S$.
While this strategy improves empirical performance, the choice of the weighting scheme is still heuristic and lacks a unified Bayesian justification.

In contrast, we retain the original KL regularization term derived in Eq.~\eqref{eq:loss}, without introducing sample-dependent masking or ad-hoc modifications.

\section{Experiments}
In this section, we present a comprehensive empirical evaluation of Generalized Evidential Deep Learning (GEDL).
Our experiments are designed to address two central questions.
First, we examine whether GEDL, despite being derived from a theoretically grounded Bayesian formulation, achieves competitive performance in standard classification, uncertainty estimation, and out-of-distribution detection when compared with existing evidential and Dirichlet-based methods.
Second, we investigate whether the uncertainty estimated by GEDL is consistent with established notions of distributional uncertainty, thereby empirically supporting our theoretical analysis.
The details about dataset and implementation settings can be found in the Appendix.~\ref{apdx:exp_set}.

\subsection{Classification}

\paragraph{Baselines.}
Following prior work~\cite{C_2023ICLR_IEDL,C_2024ICLR_REDL}, we compare GEDL against a broad set of uncertainty-aware classification methods.
These include Dirichlet-distribution-based approaches, namely PN~\citep{C_2018NIPS_prior_net}, RKL-PN~\citep{C_2019NIPS_RPN}, and PostNet~\citep{C_2020NIPS_post_net}.
We also include evidential learning methods, including vanilla EDL~\citep{C_2018NIPS_EDL}, $\mathcal{I}$-EDL~\citep{C_2023ICLR_IEDL}, and R-EDL~\citep{C_2024ICLR_REDL}.
This selection covers both classical Dirichlet-based uncertainty models and recent evidential variants.
For fairness and reproducibility, all baseline results are obtained from the original articles or by evaluating publicly released checkpoints.
RED~\cite{C_2023ICML_RED} is not included, as no pretrained checkpoints or corresponding evaluation results on uncertainty estimation and OOD detection are provided.

\paragraph{Evaluation metrics.}
We evaluate model performance from three complementary perspectives: classification performance, uncertainty estimation, and OOD detection.
For in-distribution classification, we report standard top-1 accuracy.
To assess uncertainty estimation, we follow established practice in~\cite{C_2023ICLR_IEDL,C_2024ICLR_REDL} and measure confidence quality using {AUPR}.
Specifically, we use two confidence scores: the maximum predictive probability (MP), and the uncertainty mass (UM).
The uncertainty mass in GEDL is defined under the subjective opinion formulation as as $u = W / S$, whereas other approaches compute it from the Dirichlet concentration $S=\sum_k \alpha_k$.
For OOD detection, these confidence scores are used to discriminate in-distribution and out-of-distribution samples, and performance is also evaluated using AUPR, with in-distribution samples labeled as positive and OOD samples as negative.
For MC Dropout and DUQ, which do not rely on Dirichlet distributions, we report results based solely on MP.

\paragraph{Results.}
Table~\ref{tab:results} reports the performance of GEDL and competing methods on both in-distribution classification and out-of-distribution detection.
Overall, GEDL consistently achieves performance that is comparable to or better than existing evidential and Dirichlet-based baselines across all evaluated settings.
For ID classification, GEDL attains the highest accuracy on both MNIST and CIFAR10, demonstrating that enforcing a principled Bayesian formulation does not compromise predictive performance. Instead, the structured treatment of prior, posterior update, and training objective yields models that remain highly competitive in standard classification tasks.

More importantly, GEDL shows clear advantages in uncertainty-based OOD detection. 
In each setting, GEDL achieves the best or second-best performance under both MP and UM metrics. 
The improvements become more pronounced on the more challenging CIFAR10 benchmarks. In particular, for CIFAR10$\rightarrow$SVHN, GEDL outperforms the strongest baseline (R-EDL) by $5.01\%$ in MP and $8.89\%$ in UM, indicating substantially improved separation between in-distribution and out-of-distribution samples.

Overall, GEDL achieves a favorable balance between predictive accuracy and uncertainty quality, providing both strong empirical performance and a clearer theoretical grounding for uncertainty-aware classification.
Ablation studies and parameter sensitivity analysis for $C_w$ and $C_{\tau}$ can be found in Appendix.~\ref{apdx:addition_results}.

\subsection{Evaluation of Distributional Uncertainty}

To empirically validate our theoretical claim that evidential uncertainty in EDL reflects distributional uncertainty, we compare the learned uncertainty mass with two Bayesian measures of distributional uncertainty defined in Eq.~\eqref{eq:dist_MI} and Eq.~\eqref{eq:dist_var_sum}.
The quantitative results are presented in Fig.~\ref{fig:uncertainty}.

In the top row, we illustrate the relationship between different uncertainty measures and the concentration parameter $\log S$ on clean test data, using both rank-based and value-based visualizations (the left two panels correspond to MNIST and the right two to CIFAR-10). 
All three uncertainty measures exhibit highly consistent monotonic trends with respect to $\log S$, indicating that the evidential uncertainty $u = W/S$ exhibits behavior that is closely aligned with Bayesian distributional uncertainty.

In the bottom row, we examine the behavior of these uncertainty measures under increasing corruption severity on CIFAR-10-C~\cite{A_2019_cifar10c}, with snow, brightness, pixelate, and shot noise, respectively.
The three distributional uncertainty metrics show similar growth patterns as the corruption level increases, reflecting increased ambiguity in the latent class probabilities. In contrast, the epistemic uncertainty estimated via ensemble methods (dashed light blue curve) exhibits a qualitatively different trend, highlighting that it captures model uncertainty rather than distributional uncertainty.

Together, these results provide empirical support for our theoretical analysis: the evidential uncertainty mass produced by GEDL aligns with Bayesian notions of distributional uncertainty.

\section{Conclusion}
In this work, we provide a principled Bayesian interpretation of Evidential Deep Learning and show that its core mechanisms can be understood within a generalized Bayesian framework.
Moreover, we demonstrate the uncertainty mass estimated by EDL aligns with Bayesian notions of distributional uncertainty, both theoretically and empirically.

Based on this insight, we propose Generalized Evidential Deep Learning (GEDL), a unified framework that replaces empirical design choices with theoretically grounded constructions. GEDL yields competitive or improved performance in classification, uncertainty calibration, and out-of-distribution detection.

Overall, our work reformulates EDL as a coherent instance of generalized Bayesian inference. We believe this framework opens new directions for principled uncertainty modeling and provides a foundation for designing more reliable and interpretable uncertainty-aware deep learning systems.

\section*{Impact Statement}
This paper presents work whose primary goal is to advance the theoretical understanding of evidential deep learning and uncertainty estimation. Improved uncertainty modeling could contribute to safer and more reliable machine learning systems in applications where confidence assessment is important. We do not foresee direct negative societal consequences specific to this work beyond the general considerations applicable to machine learning research.

\nocite{langley00}

\bibliography{strings,refs}
\bibliographystyle{icml2026}

\newpage
\appendix
\onecolumn
\section{Notations}\label{sec:notation}

\begin{table*}[htbp]
\centering
\caption{Summary of notations used throughout the paper.}
\renewcommand{\arraystretch}{1.2}
\begin{tabular}{llll}
\hline
\textbf{Category} & \textbf{Symbol} & \textbf{Index / Dimension} & \textbf{Description} \\
\hline

Data & $\mathcal{D}$ & - & Dataset \\
Data & $\bm{x}$ & $ \in \mathcal{X}$ & Input sample \\
Data & $y$ & $\in \{1,\dots,K\}$ & Class label \\
Data & $\bm{y}$ & $\in \{0,1\}^K$ & One-hot encoded label vector \\

Latent variable & $\bm{\pi}$ & $\in \Delta^{K-1}$ & Latent categorical probability vector \\
Latent variable & $\pi_k$ & $k=1,\dots,K$ & Probability of class $k$ \\

Model parameter & $\bm{\theta}$ & -- & Neural network parameters \\
Model parameter & $f_{\bm{\theta}}(\cdot)$ & -- & Neural network mapping from input to evidence \\

Dirichlet parameter & $\bm{\alpha}$ & $\in \mathbb{R}_{>0}^K$ & Dirichlet concentration parameters (posterior) \\
Dirichlet parameter & $\alpha_k$ & $k=1,\dots,K$ & $k$-th concentration parameter \\
Dirichlet parameter & $\bm{\alpha}_0$ & $\in \mathbb{R}_{>0}^K$ & Prior Dirichlet concentration parameters \\
Dirichlet parameter & $S$ & scalar & Total concentration: $S=\sum_k \alpha_k$ \\

Evidence & $\bm{e}(\bm{x})$ & $\in \mathbb{R}_{\ge 0}^K$ & Evidence vector predicted by the network \\
Evidence & $E(x)$ & scalar & Evidence strength: $E(x)=\|\bm{e}(x)\|$ \\
Evidence & $\bm{r}(x)$ & $\in \Delta^{K-1}$ & Evidence direction: $\bm{r}(x)=\bm{e}(x)/\|\bm{e}(x)\|$ \\
Uncertainty & $u$ & scalar & Uncertainty mass: $u = W / S(x)$ \\
Belief & $\bm{b}$ & $\in \mathbb{R}_{\ge 0}^{K}$ & Normalized belief mass derived from evidence \\
Belief & $b_k$ & $k=1,\cdots,K$ & Support the $k_{th}$ class \\
Opinion &  $ \bm {D}$ & $\in \mathbb{R}_{\ge 0}^{K+1}$ & Opintion representation $\bm{D}=\{\bm{b},u\}$ \\  

Prior & $\bm{a}$ & $\in \Delta^{K-1}$ & Base rate in subjective logic (prior probability) \\
Prior & $W$ & scalar & Prior strength \\
Prior & $\bm{\alpha}_0 = W\bm{a}$ & $\in \mathbb{R}_{>0}^K$ & Dirichlet prior parameters \\

Posterior & $q_\theta(\bm{\pi}\mid x)$ & -- & Variational posterior distribution \\
Posterior & $q_{\tau}(\bm{\pi}\mid y)$ & -- & Generalized Bayesian posterior \\
Parameter & $\tau$ & scalar & Evidence strength in generalized Bayesian update \\
Parameter & $\lambda$ & scalar & KL regularization weight ($\lambda = 1/\tau$) \\

Functions & $\psi(\cdot)$ & -- & Digamma function \\
Sets & $\Delta^{K-1}$ & -- & Probability simplex in $\mathbb{R}^K$ \\
Sets & $K$ & scalar & Number of classes \\

\hline
\end{tabular}
\label{tab:notation}
\end{table*}

\section{Technical Derivations and Proofs}
\raggedbottom 

\begin{proposition}[Conjugacy of the Dirichlet Distribution]\label{prop:conj_dir}
Let $\bm{\pi} = (\pi_1,\dots,\pi_K)$ be a probability vector on the $K$-simplex, and let the prior distribution of $\bm{\pi}$ be a Dirichlet distribution with parameter $\bm{\alpha}_0 = (\alpha_{0, 1},\dots,\alpha_{0, K})$:
$$
p(\bm{\pi} \mid \bm{\alpha}_0) = \mathrm{Dir}(\bm{\pi} \mid \bm{\alpha}_0)
= \frac{1}{\mathrm{B}(\bm{\alpha}_0)} \prod_{k=1}^K \pi_k^{\alpha_{0, k}-1}.
$$

Suppose that conditional on $\bm{\pi}$, we observe data generated from a categorical (or multinomial) distribution with count vector
$\bm{n} = (n_1,\dots,n_K)$, where $n_k$ denotes the number of observations in category $k$.

Then the posterior distribution of $\bm{\pi}$ is also Dirichlet, with parameter $\bm{\alpha} = \bm{\alpha}_0 + \bm{n}$, that is,
$$
p(\bm{\pi} \mid \bm{n}, \bm{\alpha}_0) = \mathrm{Dir}(\bm{\pi} \mid \alpha_{0, 1} + n_1,\dots,\alpha_{0,K} + n_K).
$$
\end{proposition}

\begin{proof}
By Bayes' theorem,
$$
p(\bm{\pi} \mid \bm{n}, \bm{\alpha}_0) \propto p(\bm{n} \mid \bm{\pi})\, p(\bm{\pi} \mid \bm{\alpha}_0).
$$

The likelihood of the data (up to a normalization constant independent of $\bm{\pi}$) is
$$
p(\bm{n} \mid \bm{\pi}) \propto \prod_{k=1}^K \pi_k^{n_k}.
$$

Substituting the likelihood and the prior, we obtain
$$
p(\bm{\pi} \mid \bm{n}, \bm{\alpha}_0)
\propto
\left( \prod_{k=1}^K \pi_k^{n_k} \right)
\left( \prod_{k=1}^K \pi_k^{\alpha_{0, k}-1} \right).
$$

Combining terms yields
$$
p(\bm{\pi} \mid \bm{n}, \bm{\alpha}_0)
\propto
\prod_{k=1}^K \pi_k^{\alpha_{0, k} + n_k - 1},
$$
which is the kernel of a Dirichlet distribution with parameter
$\bm{\alpha} = \bm{\alpha}_0 + \bm{n}$.

\end{proof}

\begin{proposition}[Conjugacy of the Dirichlet Distribution in Generalized Bayesian Inference]
\label{prop:gen_conj_dir}

Consider a $K$-class categorical observation represented by a one-hot vector
$\bm{y}=(y_1,\ldots,y_K)$, where $y_k\in\{0,1\}$ and $\sum_{k=1}^K y_k=1$.
Let the latent class probability vector be $\bm{\pi}=(\pi_1,\ldots,\pi_K)\in\Delta^{K-1}$,
and assume a Dirichlet prior
$$
p(\bm{\pi})=\mathrm{Dir}(\bm{\pi}\mid\bm{\alpha}_0), \qquad 
\bm{\alpha}_0\in\mathbb{R}_{>0}^K.
$$
Under the tempered generalized Bayesian update with temperature $\tau>0$, the posterior is
$$
q_\tau(\bm{\pi}\mid \bm{y}) =\frac{p(\bm{y}\mid\bm{\pi})^{\tau}\,p(\bm{\pi})}{\int_{\Delta^{K-1}} p(\bm{y}\mid\bm{\pi})^{\tau}\,p(\bm{\pi})\,\mathrm{d}\bm{\pi}}= \mathrm{Dir}\big(\bm{\pi}\mid \bm{\alpha}_0 + \tau \bm{y}\big).
$$

More generally, for a multinomial count vector $\bm{n}\in\mathbb{R}_{\ge 0}^K$,
the tempered posterior is
$$
q_\tau(\bm{\pi} \mid \bm{n})
= \mathrm{Dir}\big(\bm{\pi}\mid \bm{\alpha}_0 + \tau \bm{n}\big).
$$
\end{proposition}

\begin{proof}
We start from the generalized Bayesian posterior definition with temperature $\tau>0$:
\begin{equation}
q_\tau(\bm{\pi}\mid \bm{y})
= \frac{p(\bm{y}\mid\bm{\pi})^{\tau}\,p(\bm{\pi})}{\int_{\Delta^{K-1}} p(\bm{y}\mid\bm{\pi})^{\tau}\,p(\bm{\pi})\,d\bm{\pi}}.
\label{eq:gen_bayes_def}
\end{equation}

For a categorical observation encoded as a one-hot vector $\bm{y}$, the power $\tau$ gives likelihood term as:
\begin{equation*}
p(\bm{y}\mid\bm{\pi})^{\tau}
=\left(\prod_{k=1}^K \pi_k^{y_k}\right)^{\tau}
=\prod_{k=1}^K \pi_k^{\tau y_k}.
\end{equation*}

For the prior term, the Dirichlet density is
\begin{equation*}
p(\bm{\pi})
=\mathrm{Dir}(\bm{\pi}\mid \bm{\alpha}_0)=\frac{\prod_{k=1}^K \Gamma(\alpha_{0,k})}{\Gamma\!\left(\sum_{k=1}^K \alpha_{0,k}\right)}\prod_{k=1}^K \pi_k^{\alpha_{0,k}-1}
=\frac{1}{\mathrm{B}(\bm{\alpha}_0)}\prod_{k=1}^K \pi_k^{\alpha_{0,k}-1},
\end{equation*}
where $\mathrm{B}(\bm{\alpha}_0)$ is the multivariate Beta function.

Therefore,
\begin{align}
q_\tau(\bm{\pi}\mid \bm{y}) 
&=\frac{\frac 1 {\mathrm{B}(\bm{\alpha}_0)}\prod_{k=1}^K \pi_k^{\alpha_{0,k}+\tau y_k-1}}{\frac 1 {\mathrm{B}(\bm{\alpha}_0)}\int_{\Delta^{K-1}} \prod_{k=1}^K \pi_k^{\alpha_{0,k}+\tau y_k-1}\,\mathrm{d}\bm{\pi}}\\
&=\frac{\frac 1 {\mathrm{B}(\bm{\alpha}_0)}\prod_{k=1}^K \pi_k^{\alpha_{0,k}+\tau y_k-1}}{\frac 1 {\mathrm{B}(\bm{\alpha}_0)}\mathrm{B}(\bm{\alpha}_0+\tau \bm{y})}\\
&=\frac{1}{\mathrm{B}(\bm{\alpha}_0+\tau \bm{y})} \prod_{k=1}^K \pi_k^{\alpha_{0,k}+\tau y_k-1},
\end{align}
which is exactly the form of $\mathrm{Dir}\big(\bm{\pi}\mid \bm{\alpha}_0 + \tau \bm{y}\big)$.

Similarly, for a count vector $\bm{n} \in \mathbb{R}_{\ge 0}^K$, where $\sum_k n_k = N$, the multinomial likelihood is given by 
\[
p(\bm{n} \mid \bm{\pi}) \propto \prod_{k=1}^K \pi_k^{n_k}.
\]
Raising it to $\tau$ gives $\prod_k \pi_k^{\tau n_k}$. Repeating the same steps and replacing $\tau \bm{y}$ with $\tau \bm{n}$, we get
$q_\tau(\bm{\pi} \mid \bm{n}) = \mathrm{Dir}(\bm{\alpha}_0 + \tau \bm{n})$.
\end{proof}

\begin{proposition}[Expectations of Dirichlet Distribution]\label{prop:expectation}
Let  $ \psi(x) = \frac{d}{dx} \log \Gamma(x) $ denote the digamma function. For a parameter vector $\bm{\alpha} = (\alpha_1, \dots, \alpha_K) \in \mathbb{R}_{>0}^K$, define $S = \sum_{k=1}^K \alpha_k$. For a Dirichlet distribution $ p(\bm{\pi}) = \mathrm{Dir}(\bm{\pi} \mid \bm{\alpha}) $  with parameters $\bm{\alpha} \in \mathbb{R}_{>0}^K$, the following identities hold:
\begin{itemize}
    \item $\mathbb{E}[\pi_k] = \frac{\alpha_k}{S}$
    \item $\mathrm{Var}[\pi_k] = \frac{\alpha_k(S-\alpha_k)}{S^2(S+1)}$
    \item $\mathbb{E}[\log \pi_k] = \psi(\alpha_k)-\psi(S)$
    \item $\mathbb{E}[\pi_k\log \pi_k] = \frac{\alpha_k}{S}\left(\psi(\alpha_k+1)-\psi(S+1)\right)$
\end{itemize}
\end{proposition}

\begin{proof}
First, 
\begin{equation}\label{eq:expectp}
\begin{split}
    \mathbb{E}[\pi_k]&=\int_{\Delta^{K-1}} \pi_k\frac{\Gamma(S)}{\prod_{i=1}^{K} \Gamma(\alpha_i)} \prod_{i=1}^{K} \pi_i^{\alpha_i-1}\;\mathrm{d} \bm{\pi}
    \\&=\frac{\alpha_k}{S}\int_{\Delta^{K-1}} \frac{\Gamma(S+1)}{\Gamma(\alpha_k+1)\prod_{i\neq k}\Gamma(\alpha_i)} \pi_k^{\alpha_k}\prod_{i\neq k} \pi_i^{\alpha_i-1}\;\mathrm{d} \bm{\pi}
    \\&=\frac{\alpha_k}{S}.
\end{split}  
\end{equation}

For variance, using the same trick of expectation of $\pi_k$ we can get 
$$\mathbb E[\pi_k^2]=\frac{\alpha_k(\alpha_k+1)}{S(S+1)}.$$
Thus:
\begin{equation}\label{eq:expectvar}
{\rm Var}[\pi_k]=\mathbb E[\pi_k^2]-\mathbb E[\pi_k]^2=\frac{\alpha_k(\alpha_k+1)}{S(S+1)}-\frac{\alpha_k^2}{S^2}=\frac{\alpha_k(S-\alpha_k)}{S^2(S+1)} .   
\end{equation}

Then for the log-term, we have:
$$
\mathrm{B}(\bm{\alpha})=\int_{\Delta^{K-1}} \exp{\sum_{i=1}^K (\alpha_i-1)\log\pi_i}\;\mathrm{d}\bm{\pi}.
$$    
Differentiating both sides and multiplying by $\frac 1{{\rm}B(\bm{\alpha})}$ yields:
\begin{equation*}
\begin{split}
\frac{\partial}{\partial \alpha_k} \log \mathrm{B}&=
\frac 1{\mathrm{B}(\bm{\alpha})}\frac{\partial \mathrm{B}(\bm{\alpha})}{\partial \alpha_k}
\\&=\int_{\Delta^{K-1}} \frac 1 {\mathrm{B}(\bm{\alpha})}\log\pi_k\exp{\sum_{i=1}^K (\alpha_i-1)\log\pi_i}\;\mathrm{d}\bm{\pi}
\\&=\mathbb{E}_{\mathrm{Dir}(\bm{\pi} \mid \bm{\alpha})}[\log \pi_k].
\end{split}
\end{equation*}

Combined with $\log\mathrm{B}(\bm{\alpha})=\sum_{i=1}^{K} \log\Gamma(\alpha_i)-\log \Gamma(S) $, we can get $\frac{\partial}{\partial \alpha_k} \log \mathrm{B}(\bm{\alpha}) = \psi(\alpha_k) - \psi(S).$ Therefore, 
\begin{equation}\label{eq:expectlog}
  \mathbb{E}[\log \pi_k]=\frac{\partial \log \mathrm{B}(\bm{\alpha})}{\partial \alpha_k} = \psi(\alpha_k)-\psi(S).  
\end{equation}

Finally for $\mathbb E[\pi_k\log\pi_k]$:
\begin{equation}\label{eq:expecth}
    \begin{split}
        \mathbb E[\pi_k\log\pi_k]&=\int_{\Delta^{K-1}} \pi_k\log\pi_k\frac{\Gamma(S)}{\prod_{i=1}^{K} \Gamma(\alpha_i)} \prod_{i=1}^{K} \pi_i^{\alpha_i-1}\;\mathrm{d} \bm{\pi}\\
        &=\frac{\alpha_k}{S}\int_{\Delta^{K-1}} \log\pi_k\frac{\Gamma(S+1)}{\Gamma(\alpha_k+1)\prod_{i\neq k} \Gamma(\alpha_i)} \pi_k^{\alpha_k}\prod_{i\neq k} \pi_i^{\alpha_i-1}\;\mathrm{d} \bm{\pi}\\
        &=\frac{\alpha_k}{S}\left(\psi(\alpha_k+1)-\psi(S+1)\right).
    \end{split}
\end{equation}

\end{proof}

\begin{proposition}[Detailed derivation of the training objective]\label{prop:elbo}
Let  $q_{\theta}(\bm{\pi} \mid \bm{x})$  be a variational posterior distribution of $\bm{\pi}$ ,  $p(\bm{\pi})$ be its prior distribution, and $q_\tau(\bm{\pi} \mid \bm{y})$ be its tempered posterior, which is defined as
$$
q_\tau(\bm{\pi} \mid \bm{y}) = \frac{1}{Z} \, p(\bm{\pi}) \, p(\bm{y} \mid \bm{\pi})^\tau,
$$ 
where $\tau > 0$ is a scaling parameter and $Z = \int p(\bm{\pi}) p(\bm{y} \mid \bm{\pi})^\tau \,\mathrm{d}\bm{\pi}$ is the normalizing constant. Then, the KL divergence between $q_{\theta}(\bm{\pi}\mid \bm{x})$  and $q_\tau(\bm{\pi} \mid \bm{y})$ decomposes as:
\begin{equation*}
\begin{split}
    \operatorname{KL}\!\left(
        q_{\theta}(\bm{\pi}\mid \bm{x})
        \;\|\;
        q_\tau(\bm{\pi} \mid \bm{y})
    \right)
    &= \operatorname{KL}\!\left(q_{\theta}(\bm{\pi}\mid \bm{x})\;\|\;p(\bm{\pi})\right) \\
    &\quad + \tau\,\mathbb{E}_{q_\theta(\bm{\pi}\mid \bm{x})}\!\big[-\log p(\bm{y}\mid \bm{\pi})\big] + \log Z.
\end{split}
\end{equation*}
Moreover, since $Z$ does not depend on the variational parameters $\theta$, it is a constant with respect to the optimization objective.
\end{proposition}

\begin{proof}
By definition of the KL divergence, 
$$
\operatorname{KL}\!\left(q_{\theta} \,\|\, q_\tau \right) 
= \mathbb{E}_{q_{\theta}} \left[ \log \frac{q_{\theta}(\bm{\pi} \mid \bm{x})}{q_\tau(\bm{\pi} \mid \bm{y})} \right].
$$

Substituting the expression for  $ q_\tau(\bm{\pi} \mid \bm{y}) $ , we obtain
$$
\begin{aligned}
\operatorname{KL}\!\left(q_{\theta} \,\|\, q_\tau \right) 
&= \mathbb{E}_{q_{\theta}} \left[ \log q_{\theta}(\bm{\pi} \mid \bm{x}) - \log p(\bm{\pi}) - \tau \log p(\bm{y} \mid \bm{\pi}) + \log Z \right] \\
&= \underbrace{\mathbb{E}_{q_{\theta}} \left[ \log \frac{q_{\theta}(\bm{\pi} \mid \bm{x})}{p(\bm{\pi})} \right]}_{\operatorname{KL}(q_{\theta} \| p)} 
+ \tau \, \mathbb{E}_{q_{\theta}} \left[ -\log p(\bm{y} \mid \bm{\pi}) \right] 
+ \log Z,
\end{aligned}
$$ 
which yields the desired result. Since  $ Z $  depends only on the data $ \bm{y} $ , the prior  $ p(\bm{\pi}) $ , and the fixed scale  $ \tau $ , it is independent of the variational parameters  $ \theta $  and can be omitted during optimization. The final objective, after scaling the KL divergence by $1/\tau$, is the variational loss:
\[
\mathcal{L}_{\mathrm{VI}} = \mathbb{E}_{q_\theta(\bm{\pi} \mid \bm{x})}\!\big[-\log p(\bm{y} \mid \bm{\pi})\big] + \frac{1}{\tau} \operatorname{KL}\!\big(q_\theta(\bm{\pi} \mid \bm{x}) \;\|\; p(\bm{\pi})\big).
\]
\end{proof}

\begin{proposition}[Closed-form expression of the training objective in the Dirichlet distribution]\label{prop:closed-form}
Let $\bm{y}$ be a one-hot vector. Denote $p(\bm{y} \mid \bm{\pi}) = \pi_y;\, \bm{\pi} \sim \mathrm{Dir}(\bm{\pi}\mid\bm{\alpha});\, S=\sum \alpha_k$, and $\alpha_y$ denote the component of $\bm{\alpha}$ corresponding to the non-zero entry of the one-hot vector $\bm{y}$. Then we have
$$
\mathbb{E}_{\bm{\pi}}[-\log p(y \mid \bm{\pi})] = \psi(S) -\psi(\alpha_y).
$$ 

As for the Kullback-Leibler Divergence term. Let $q(\bm{\pi}) = \mathrm{Dir}(\bm{\pi} \mid \bm{\alpha})$ and $p(\bm{\pi}) = \mathrm{Dir}(\bm{\pi} \mid \bm{\alpha}_0)$ be two Dirichlet distributions with parameters  $ \bm{\alpha}, \bm{\alpha}_0 \in \mathbb{R}_{>0}^K $ . Then we have
\begin{equation*}
\operatorname{KL}(q \,\|\, p) 
= \log \frac{\mathrm{B}(\bm{\alpha}_0)}{\mathrm{B}(\bm{\alpha})}
+ \sum_{k=1}^K (\alpha_k - \alpha_{0,k}) \big( \psi(\alpha_k) - \psi(S) \big),
\end{equation*}
where $\mathrm{B}(\bm{\alpha}) = \frac{\prod_{k=1}^K \Gamma(\alpha_k)}{\Gamma(S)}$  is the multivariate Beta function.

\end{proposition}

\begin{proof}
The data fitting term follows from the Dirichlet log-expectation identity $\mathbb E[\log \pi_k]=\psi(\alpha_k)-\psi(S)$ (see Eq.~\eqref{eq:expectlog}), where $k$ is the active index of the one-hot label vector $\bm{y}$.

Then, the KL divergence is defined as
$$
\operatorname{KL}(q \,\|\, p) = \mathbb{E}_{q}[\log q(\bm{\pi}) - \log p(\bm{\pi})].
$$
Using the log-density of the Dirichlet distribution,
\begin{equation*}
\mathbb E_q \log \mathrm{Dir}(\bm{\pi} \mid \bm{\alpha}) =-\log \mathrm{B}(\bm{\alpha}) + \sum_{k=1}^K (\alpha_k - 1) \mathbb E_q [\log \pi_k],    
\end{equation*}

Similarly, for  $ \mathbb{E}_q[\log p] $,
\begin{equation*}
\mathbb E_q \log \mathrm{Dir}(\bm{\pi} \mid \bm{\alpha}_0) =-\log \mathrm{B}(\bm{\alpha}_0) + \sum_{k=1}^K (\alpha_{0,k} - 1) \mathbb E_q [\log \pi_k].
\end{equation*}

Subtracting yields the result after simplification using the property $\mathbb E[\log \pi_k]=\psi(\alpha_k)-\psi(S).$ (see Eq.~\eqref{eq:expectlog})
\begin{equation}
    \begin{split}
    \operatorname{KL}(q \,\|\, p)& =\log \frac{\mathrm{B}(\bm{\alpha}_0)}{\mathrm{B}(\bm{\alpha})}+\sum_{k=1}^K(\alpha_k-\alpha_{0,k})\left(\mathbb E_q[\log\pi_k]\right)\\
    &=\log \frac{\mathrm{B}(\bm{\alpha}_0)}{\mathrm{B}(\bm{\alpha})}+\sum_{k=1}^K(\alpha_k-\alpha_{0,k})\left(\psi(\alpha_k)-\psi(S)\right).
    \end{split}
\end{equation}
\end{proof}

\begin{proposition}[Connection to Generalized Bayesian Updating]\label{prop:generalized_belief}
The variational objective  $ \mathcal{L}_{\mathrm{VI}} $  in Eq.~\eqref{eq:elbo} is a special case of the generalized posterior updating framework~\cite{J_2016_General_Bayes}, which revises beliefs by minimizing a composite loss functional of the form
$$
L(q; p, \bm{y}) =h_1(q, \bm{y}) + h_2(q, p)= w \, \mathbb{E}_{q(\bm{\pi})}[\ell(\bm{\pi}, \bm{y})] + \operatorname{KL}\!\bigl( q(\bm{\pi}) \,\|\, p(\bm{\pi}) \bigr),
$$
where $w > 0$ balances data fidelity and prior regularization.
\end{proposition}

\begin{proof}
The generalized Bayesian framework~\cite{J_2016_General_Bayes} seeks a posterior distribution that coherently incorporates observed data through a task-specific loss $ \ell(\bm{\pi}, \bm{y})$, rather than requiring a probabilistic likelihood model. The optimal posterior minimizes
$$
L(q; p, \bm{y}) = w \, \mathbb{E}_{q(\bm{\pi})}[\ell(\bm{\pi}, \bm{y})] + \operatorname{KL}\!\bigl( q(\bm{\pi}) \,\|\, p(\bm{\pi}) \bigr),
$$
and admits the Gibbs form $\hat{q}(\bm{\pi}) \propto \exp(-w \ell(\bm{\pi}, \bm{y})) p(\bm{\pi})$.

In our setting, we choose the loss function as the negative log-likelihood, $\ell(\bm{\pi}, \bm{y}) = -\log p(\bm{y} \mid \bm{\pi})$, and restrict the posterior to a parametric family $q_\theta(\bm{\pi} \mid \bm{x})$ conditioned on input $\bm{x}$. Putting coefficient $\lambda$ on the KL term, our training objective becomes
$$
\min_{\theta} \;\mathcal{L}_{\mathrm{VI}}=\mathbb{E}_{q_\theta(\bm{\pi}\mid \bm{x})}\!\bigl[ -\log p(\bm{y} \mid \bm{\pi}) \bigr] +\lambda \operatorname{KL}\!\bigl( q_\theta(\bm{\pi}\mid \bm{x}) \,\|\, p(\bm{\pi}) \bigr),
$$ 
which is precisely an instantiation of the general functional $L(q; p, \bm{y})$.
\end{proof}

\begin{proposition}[MSE-loss]\label{prop:MSE-loss}
The training objective in vanilla Evidential Deep Learning~\cite{C_2018NIPS_EDL}
$$
\mathcal{L} = \sum_{i=1}^{N} \mathbb{E}_{\bm{\pi} \sim \mathrm{Dir}(\bm{\pi} \mid \bm{\alpha}(\mathbf{x}^{(i)}))} \bigl[ \|\bm{y}^{(i)} - \bm{\pi}\|^2 \bigr] 
+ \lambda_t \sum_{i=1}^{N} \operatorname{KL}\!\left[ \mathrm{Dir}(\bm{\pi}_i \mid \tilde{\bm{\alpha}}_i) \,\middle\|\, \mathrm{Dir}(\bm{\pi}_i \mid \mathbf{1}) \right],
$$ 
can be viewed as a variant of our variational inference objective under generalized posterior updating.
\end{proposition}

\begin{proof}
Under the generalized (Gibbs) posterior framework, suppose the loss function is defined as 
$$
\ell(\bm{\pi},\bm{y}) = \|\bm{y} - \bm{\pi}\|^2.
$$
Or equivalently, consider a Gaussian likelihood model $p(\bm{y}\mid\bm{\pi})=\mathcal N(\bm{\pi},\sigma^2 I)$, then we use the negative log-likelihood loss function again and get $\ell(\bm{\pi},\bm{y})=\|\bm{y}-\bm{\pi}\|^2$. 
 
Then the Gibbs posterior takes the form:
$$
q_{\lambda}(\bm{\pi}) = \frac{1}{Z} \exp\!\left( -\frac{1}{\lambda} \ell(\bm{\pi},\bm{y}) \right) p(\bm{\pi}),
$$
where  $ Z $  is the normalization constant and  $ \lambda > 0 $  acts as an inverse temperature (annealing) parameter.

From the perspective of variational inference, we aim to minimize the KL divergence between an approximate posterior  $ q_{\theta}(\bm{\pi}) $  and the true Gibbs posterior  $ q_{\lambda}(\bm{\pi}) $ . Following the derivation in Prop.~\ref{prop:expectation}, we have:
$$
\operatorname{KL}\bigl( q_{\theta}(\bm{\pi}) \,\|\, q_{\lambda}(\bm{\pi}) \bigr)
= \operatorname{KL}\bigl( q_{\theta}(\bm{\pi}) \,\|\, p(\bm{\pi}) \bigr) + \frac{1}{\lambda} \, \mathbb{E}_{q_{\theta}}[\ell(\bm{\pi},\bm{y})] + \log Z.
$$
 
Discarding the constant term  $ \log Z $  (independent of  $ \theta $ ) and multiplying the entire expression by  $ \lambda $ , the optimization objective becomes:
$$
\mathcal{L} = \mathbb{E}_{\bm{\pi} \sim q_{\theta}(\bm{\pi}\mid \bm{x})} \bigl[ \|\bm{y} - \bm{\pi}\|^2 \bigr] 
+ \lambda \, \operatorname{KL}\!\left( q_{\theta}(\bm{\pi}\mid \bm{x}) \,\middle\|\, p(\bm{\pi}) \right).
$$

Now, if we fix the prior as  $ p(\bm{\pi}) = \mathrm{Dir}(\bm{\pi} \mid \mathbf{1}) $  and heuristically set the approximate posterior to  $ q_{\theta}(\bm{\pi}) = \mathrm{Dir}(\bm{\pi} \mid \tilde{\bm{\alpha}}(\mathbf{x})) $ , where  $ \tilde{\bm{\alpha}} = \bm{y} + (\mathbf{1} - \bm{y}) \odot \bm{\alpha} $  (with  $ \odot $  denoting element-wise multiplication), then summing over the training set yields
$$
\mathcal{L} = \sum_{i=1}^{N} \mathbb{E}_{\bm{\pi} \sim \mathrm{Dir}(\bm{\pi} \mid \bm{\alpha}(\mathbf{x}^{(i)}))} \bigl[ \|\bm{y}^{(i)} - \bm{\pi}\|^2 \bigr] 
+ \lambda_t \sum_{i=1}^{N} \operatorname{KL}\!\left[ \mathrm{Dir}(\bm{\pi}_i \mid \tilde{\bm{\alpha}}_i) \,\middle\|\, \mathrm{Dir}(\bm{\pi}_i \mid \mathbf{1}) \right],
$$
which recovers the vanilla EDL objective.
\end{proof}

\begin{proposition}[Relationship to PAC Bound]\label{prop:PACbound}
The form of the regularizer in $\mathcal L_{\text{VI}}$ is structurally consistent with a PAC-Bayes bound.
\end{proposition}

\begin{proof}
In the PAC framework~\cite{C_2020NIPS_PAC}, we have
\begin{equation}\label{eq:PACboundl}
\mathbb{E}_{q(\theta)}[L(\theta)] \leq \mathbb{E}_{q(\theta)}[\hat{L}(\theta \mid \mathcal{D})] + \lambda \Bigl( \operatorname{KL}(q \| p) + f_{p,\nu}(\lambda) - \log \delta \Bigr),
\end{equation}

which holds with probability at least $1 - \delta$, and
$$
f_{p,\nu}(\lambda) = \log \mathbb{E}_{\mathcal{D} \sim \nu(x)}\mathbb{E}_{p(\theta)} \left[ \exp\left\{ \frac{1}{\lambda} \bigl( L(\theta) - \hat{L}(\theta \mid \mathcal{D}) \bigr) \right\} \right].
$$

We now prove this inequality. First, let $\phi(\theta)$ be an arbitrary function of $\theta$. Then we have
$$
\begin{aligned}
\operatorname{KL}(q \| p) + \log \mathbb{E}_{p(\theta)} \big[ \exp \phi(\theta) \big]
&= \mathbb{E}_{q(\theta)} \left[ \log \frac{q(\theta)}{p(\theta) / \mathbb{E}_{p(\theta)}[\exp \phi(\theta)]} \right] \\
&= \mathbb{E}_{q(\theta)} \left[ \log q(\theta) - \log \left\{ \frac{p(\theta) \exp \phi(\theta)}{\mathbb{E}_{p(\theta)}[\exp \phi(\theta)]} \right\} + \log \exp \phi(\theta) \right] \\
&= \operatorname{KL}\left( q \,\middle\|\, \frac{p(\theta) \exp \phi(\theta)}{\mathbb{E}_{p(\theta)}[\exp \phi(\theta)]} \right) + \mathbb{E}_{q(\theta)} \big[ \phi(\theta) \big] \\
&\geq \mathbb{E}_{q(\theta)} \big[ \phi(\theta) \big],
\end{aligned}
$$
where the last inequality follows from the non-negativity of KL divergence.
Now Assume $\phi(\theta) = D\bigl(L(\theta),\hat{L}(\theta \mid \mathcal{D}) \big)$, where $D$ is a distance function. Thus randomness in $\mathcal D$ is passed to $\phi(\theta)$. For a non-negative random variable $X$, by Markov's inequality we have
$$
\mathbb P(X\geq a)\leq \frac {\mathbb E[X]} a \Rightarrow \mathbb P(X\leq a)\geq 1-\frac {\mathbb E[X]}{a}
$$
Note that expectation of an exponential function is non-negative, so we can substitute $\frac{ \mathbb{E}_{\mathcal{D} \sim \nu(x)} [\mathbb{E}_{p(\theta)} \exp \phi(\theta) ] }{\delta}=a$ and $\mathbb{E}_{p(\theta)}\exp \phi(\theta)=X$. Therefore,
$$
\mathbb{P} \left( \mathbb{E}_{p(\theta)}\exp \phi(\theta) \leq \frac{ \mathbb{E}_{\mathcal{D} \sim \nu(x)} [\mathbb{E}_{p(\theta)} \exp \phi(\theta) ] }{\delta} \right) 
\geq 1 - \frac{ \mathbb{E}_{\mathcal{D} \sim \nu(x)} [ \mathbb{E}_{p(\theta)}\exp \phi(\theta) ] }{ \mathbb{E}_{\mathcal{D} \sim \nu(x)} [ \mathbb{E}_{p(\theta)}\exp \phi(\theta) ] / \delta } = 1 - \delta.
$$
Moreover, from the earlier inequality we obtain:
\begin{equation}\label{eq:PACboundg}
\mathbb{E}_{q}[\phi(\theta)] \leq \operatorname{KL}(q \| p) + \log \mathbb E_{p(\theta)}\mathbb{E}_{\mathcal{D} \sim \nu(x)} [ \exp \phi(\theta) ] - \log \delta.
\end{equation}

Now substitute $\phi(\theta) = \frac{1}{\lambda} \big( L(\theta) - \hat{L}(\theta \mid \mathcal{D}) \big)$, the inequality above is equivalent to Eq.~\eqref{eq:PACboundl} and holds with probability at least $1 - \delta$ .
In practical applications, recent works such as~\cite{C_2021_bayes_pac,C_2023ICLR_IEDL} employ a PAC-Bayes bound of the form in Eq.~\eqref{eq:PACboundg} and aim to minimize the expected risk
$$
\mathbb{E}_{(x,y) \sim \mathcal{P}} \, \mathbb{E}_{q_\theta(\bm{\pi} \mid x)} \big[ \ell(y,\bm{\pi} \mid x) \big],
$$ 
where  $ \ell(y,\bm{\pi} \mid x) $  is a loss function, \eg, $-p(y \mid \bm{\pi})$ as used in~\cite{C_2021_bayes_pac}. (Moreover, in~\cite{C_2021_bayes_pac}, the choice of $\phi(\theta)$ is based on a squared distance, leading to a regularizer involving a square-root KL divergence term.) These approaches replace the global parameter-level prior $p(\theta)$ and posterior $q(\theta)$ with sample-wise predictive distributions: a fixed prior $p(\bm{\pi})$ and a data-dependent posterior approximation $q_\theta(\bm{\pi} \mid x)$.

Since the term $f_{p,\nu}(\lambda)$ in Eq.~\eqref{eq:PACboundl} depends on the unknown true risk $L(\theta)$ and cannot be computed, it is omitted during optimization, leading to the practical objective that aligns with our variational loss $\mathcal{L}_{\mathrm{VI}}$ when $\lambda = 1/\tau$. That is,
$$-\frac 1 N\sum_{i=1}^N\mathbb E_{q_\theta(\bm{\pi}\mid \bm{x^{(i)}})}\log p(\bm{y}^{(i)}\mid\bm{\pi})+\lambda \operatorname{KL}(q_\theta(\bm{\pi}\mid \bm{x^{(i)}})\|p(\bm{\pi})).$$
\end{proof}

\begin{proposition}[Asymptotic Behavior of Distributional Uncertainty]\label{prop:u}
Assume the normalized evidence vector 
 $\bm{\alpha}/S$  converges to a fixed
   $ \bar{\bm{\pi}} \in \Delta^{K-1} $  as  $ S\to \infty $ . Then the mutual information satisfies
$$
H\!\big( \mathbb{E}[\pi_y] \big) - \mathbb{E}\big[ H(\pi_y) \big] = \frac{K - 1}{2S} + \mathcal{O}\!\left( S^{-2} \right).
$$
\end{proposition}

\begin{proof}
It is well known that $\psi(x+1)=\log x+\frac 1 {2x}+\mathcal{O}(x^{-2}).$ Or you can see it from the following rough derivation.

Using Stirling's approximation, as $x\to+\infty$ we have $$\Gamma(x+1)=\sqrt{2\pi x}x^xe^{-x}(1+\mathcal{O}(x^{-1})).$$
Taking the natural logarithm on both sides yields 
$$\log{\Gamma(x+1)}=(x+\frac 1 2)\log x-x+\frac 1 2\log(2\pi)+\mathcal{O}(x^{-1}).$$
Now, recall that the digamma function is the derivative of the log-gamma function. Differentiating the asymptotic expansion term by term, we obtain
$$\psi(x+1)=\log x+\frac 1 {2x}+\mathcal{O}(x^{-2}).$$ 

According to Eq.~\eqref{eq:expectp}, we can write down the closed form of $\mathbb E[\pi_y]$, thus we have
$$
H(\mathbb E[\pi_y])=-\sum_{k=1}^{K} \frac{\alpha_k}{S}\log\frac{\alpha_k}{S}.
$$
And according to Eq.~\eqref{eq:expecth}, we can directly get
$$
\mathbb E[H(\pi_y)]=-\sum_{k=1}^{K}\frac{\alpha_k}{S}\left(\psi(\alpha_k+1)-\psi(S+1)\right).
$$

When $\alpha_i\to \infty$, using the asymptotic expansion:
\begin{equation*}
    \begin{split}
&H\!\big( \mathbb{E}[\pi_y] \big) - \mathbb{E}\big[ H(\pi_y) \big]\\
=&-\sum_{k=1}^{K} \frac{\alpha_k}{S}\log\frac{\alpha_k}{S}+\sum_{k=1}^{K}\frac{\alpha_k}{S}\left(\log\frac{\alpha_k}{S}+\frac 1 {2S}(\frac{S}{\alpha_k}-1)+\mathcal{O}(S^{-2})\right)\\
=&\frac {K-1} {2S}+\mathcal{O}(S^{-2}).
    \end{split}
\end{equation*}
\end{proof}

\section{Experimental settings}\label{apdx:exp_set}
\subsection{Dataset}
Following the experimental setup of~\cite{C_2023ICLR_IEDL,C_2024ICLR_REDL}, we experiment with the following image classification datasets.
\begin{enumerate}
    \item \textbf{MNIST}~\cite{1998_MNIST}, \textbf{FMNIST}~\cite{A_2017_FashionMNIST}, \textbf{KMNIST}~\cite{C_2018NIPSW_KMNIST};
    \item \textbf{CIFAR-10}~\cite{2009_CIFAR10}, \textbf{SVHN}~\cite{D_2018street}, \textbf{CIFAR-100}~\cite{2009_CIFAR10}.
\end{enumerate}
Within each group, we specify the first dataset as in-distribution (ID) data, while the subsequent ones are used as out-of-distribution (OOD) data. Here is a detailed breakdown of each dataset:

The \textbf{MNIST}~\cite{1998_MNIST} dataset consists of handwritten digit images ranging from 0 to 9. Specifically, it contains 60,000 training images and 10,000 test images, each normalized to a  $ 28 \times 28 $  pixel tensor. We split the training set at a ratio of (0.8, 0.2) into training and validation sets.

The \textbf{Fashion-MNIST (FMNIST)}~\cite{A_2017_FashionMNIST} dataset is designed as a more challenging alternative to MNIST. Created by Zalando Research, FMNIST contains grayscale images of clothing items such as shirts, pants, sneakers, and bags. The dataset is structured similarly to MNIST, with 60,000 training images and 10,000 test images, each normalized to a  $ 28 \times 28 $  pixel tensor. When MNIST is used as ID data, FMNIST serves as OOD data.

\textbf{Kuzushiji-MNIST (KMNIST)}~\cite{C_2018NIPSW_KMNIST} is another ``drop-in'' alternative to MNIST, consisting of 60,000 training images and 10,000 test images of handwritten Japanese cursive hiragana characters. Like MNIST, each image is processed as a  $ 28 \times 28 $  pixel tensor. When MNIST is used as ID data, KMNIST is used as OOD data.

\textbf{CIFAR-10}~\cite{2009_CIFAR10} contains 60,000 color images of size  $ 32 \times 32 $ , distributed across 10 categories such as airplanes, birds, cats, and ships, with 6,000 images per class. Of these, 50,000 are used for training and 10,000 for testing. We further split the training set at a ratio of (0.95, 0.05) into training and validation sets.

The \textbf{Street View House Numbers (SVHN)}~\cite{D_2018street} dataset consists of house number images extracted from Google Street View. It contains 73,257 training samples and 26,032 test samples. When training on CIFAR-10, we use SVHN as OOD data.

\textbf{CIFAR-100}~\cite{2009_CIFAR10} has a structure similar to CIFAR-10 but contains 100 classes, each with 600 images (500 for training and 100 for testing). When CIFAR-10 is used as ID data, CIFAR-100 is used as OOD data. 

\textbf{CIFAR-10-C}~\cite{A_2019_cifar10c} is a corrupted version of the CIFAR-10 test set designed to evaluate model robustness under distribution shifts. It contains 15 types of common image corruptions, including noise, blur, weather, and digital distortions, each applied at five different severity levels. This results in 50,000 corrupted images in total, while preserving the original labels. CIFAR-10-C is used solely for testing and provides a standardized benchmark for measuring classification performance under systematic input degradations.

\subsection{Implementation Details}

Following~\cite{C_2023ICLR_IEDL,C_2024ICLR_REDL}, we adopt a standard experimental setup to ensure fair comparison across datasets and methods. 
The main hyperparameters and training configurations are summarized in Table~\ref{tab:implementation_details}.

\begin{table}[htbp]
\centering
\caption{Implementation details of experiments in the classical setting.}
\label{tab:implementation_details}
\begin{tabular}{l l l l l l c c c}
\toprule
ID dataset & OOD dataset & Architecture & Optimizer & Learning rate & Scheduler & Epochs & $C_{\tau}$ & $C_W$ \\
\midrule
MNIST & FMNIST \& KMNIST & ConvNet & Adam & $1\times10^{-3}$ & Cosine & 200 & 100 & 0.5 \\
CIFAR-10 & SVHN \& CIFAR-100 & VGG16 & Adam & $1\times10^{-4}$ & Cosine & 200 & 100 & 0.5 \\
\bottomrule
\end{tabular}
\end{table}

All experiments are conducted with a batch size of 128. 
Weight decay is set to $5\times10^{-4}$ for CIFAR-10 and to 0 for MNIST.
Each result is averaged over five independent runs with different random seeds.
We adopt the softplus function~\cite{C_2000NIPS_softplus} as the evidence activation to guarantee non-negativity of the predicted evidence.
All experiments are performed on a single NVIDIA A100 GPU.

OOD detection is evaluated without retraining, using uncertainty scores computed directly from models trained on in-distribution data.

For the computation of the scheduled evidence strength $\tau_t$ in Eq.~\eqref{eq:tau_schedule}, the expectation $\mathbb{E}[S(x_i)]$ is approximated by the mini-batch average of the Dirichlet strength at each training iteration.
To ensure compatibility with existing EDL variants that employ KL weights in the range $[0,1]$, we further enforce $\tau_t \ge 1$ by setting $
\tau_t = \max(1,\; \tau_t).$

For both the adaptive prior strength $W$ in Eq.~\eqref{eq:weight_update} and the scheduled evidence strength $\tau_t$ in Eq.~\eqref{eq:tau_schedule}, we apply a stop-gradient operator $\mathrm{stopgrad}(\cdot)$ to prevent gradients from propagating through these adaptive quantities.
This design ensures that the optimization of the network parameters remains consistent with the derived variational objective, and avoids introducing implicit dependencies between the prior specification and the variational posterior.

For CIFAR-100 training, to facilitate stable evidence accumulation in the early stages of optimization, we initialize the model using a non-evidential classifier trained for the first 5 epochs, and then switch to the full GEDL objective.
This warm-up strategy improves convergence and prevents premature over-regularization when evidence is still unreliable.

All models are trained and evaluated using the same data preprocessing and evaluation protocols as in~\cite{C_2023ICLR_IEDL,C_2024ICLR_REDL}.

\subsection{Source Code}
The code is available at an anonymous repository: \href{https://anonymous.4open.science/r/GEDL/}{https://anonymous.4open.science/r/GEDL/}.

\section{Additional Results}\label{apdx:addition_results}

\subsection{Ablation Studies}

Table~\ref{tab:ablation} reports an ablation study that investigates the contribution of individual components in GEDL.
We incrementally incorporate (i) KL regularization applied to all samples as derived from the variational inference formulation, (ii) an adaptive evidence strength parameter $\tau$, and (iii) an adaptive prior strength $W$.
The evaluation is conducted on CIFAR-10 as in-distribution data and SVHN as out-of-distribution data, using both classification accuracy and confidence-based uncertainty metrics.
The results demonstrate that each component yields consistent improvements in uncertainty estimation while preserving strong predictive performance, and that the full GEDL model achieves the best overall balance between accuracy and reliable uncertainty quantification.

\begin{table}[htbp!]
\centering
\caption{Ablation study of GEDL components.}
\label{tab:ablation}
\begin{tabular}{c c c c c c}
\toprule
KL (all samples) & Adaptive $\tau$ & Adaptive $W$ 
& ID-Acc (\%) & ID-Conf.MP (\%) & OOD-Conf.MP (\%) \\
\midrule
\texttimes & \texttimes & \texttimes & 89.06 & 98.83 & 85.98 \\
\checkmark & \texttimes & \texttimes & 86.83 & 97.04 & 87.33 \\
\checkmark & \checkmark & \texttimes & 89.43 & 97.72 & 89.22 \\
\midrule
\checkmark & \checkmark & \checkmark & \textbf{90.77} & \textbf{98.98} & \textbf{93.63} \\
\bottomrule
\end{tabular}
\end{table}

\subsection{Parameter Analysis}

We further investigate the sensitivity of GEDL to the two key hyperparameters introduced in our framework: the prior strength constant $C_W$ and the evidence strength constant $C_{\tau}$.

\paragraph{Sensitivity to $C_W$.}
Fig.~\ref{fig:cw_sensitivity} illustrates the effect of varying $C_W$ on in-distribution confidence estimation and out-of-distribution detection performance.
We evaluate $C_W \in \{0.3, 0.5, 0.7\}$ while keeping all other hyperparameters fixed, and report results on CIFAR-10 as in-distribution data and SVHN and CIFAR-100 as OOD datasets.

\paragraph{Sensitivity to $C_{\tau}$.}
Fig.~\ref{fig:ctau_sensitivity} presents the performance under different choices of the evidence strength constant $C_{\tau} \in \{70, 90, 110, 130, 150\}$.

\begin{figure}[htbp!]
    \centering
    \begin{subfigure}[t]{0.48\textwidth}
        \centering
        \includegraphics[width=\linewidth]{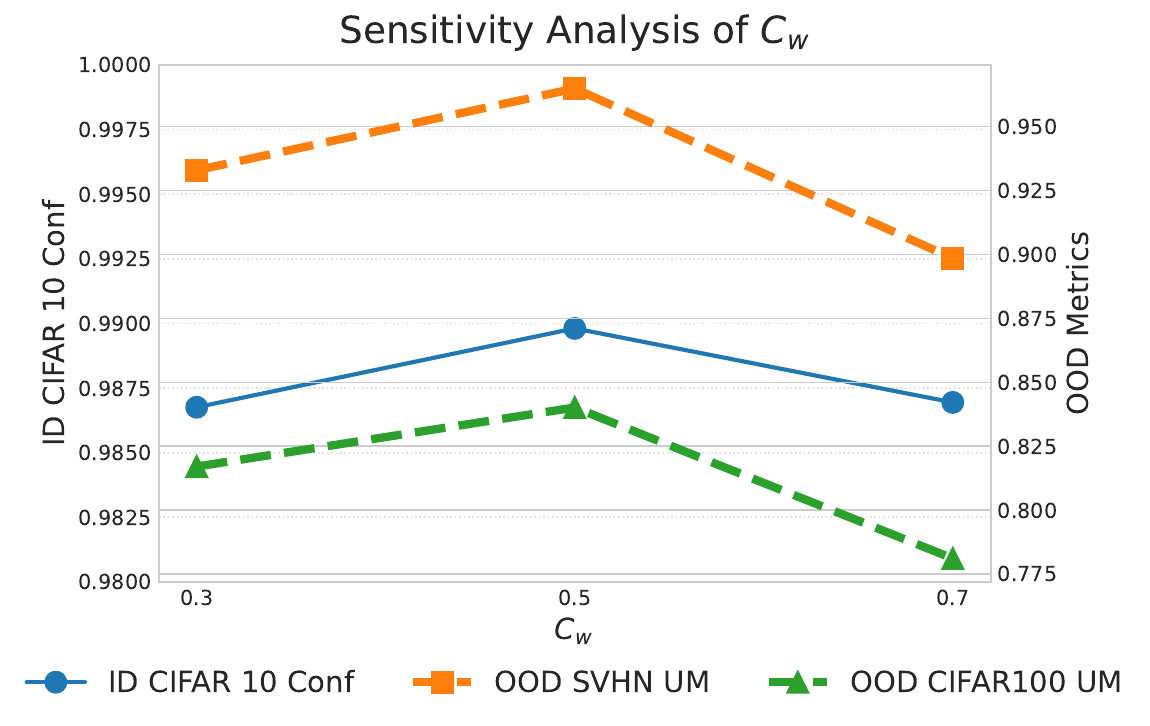}
        \caption{Sensitivity analysis w.r.t. prior strength constant $C_W$.}
        \label{fig:cw_sensitivity}
    \end{subfigure}
    \hfill
    \begin{subfigure}[t]{0.48\textwidth}
        \centering
        \includegraphics[width=\linewidth]{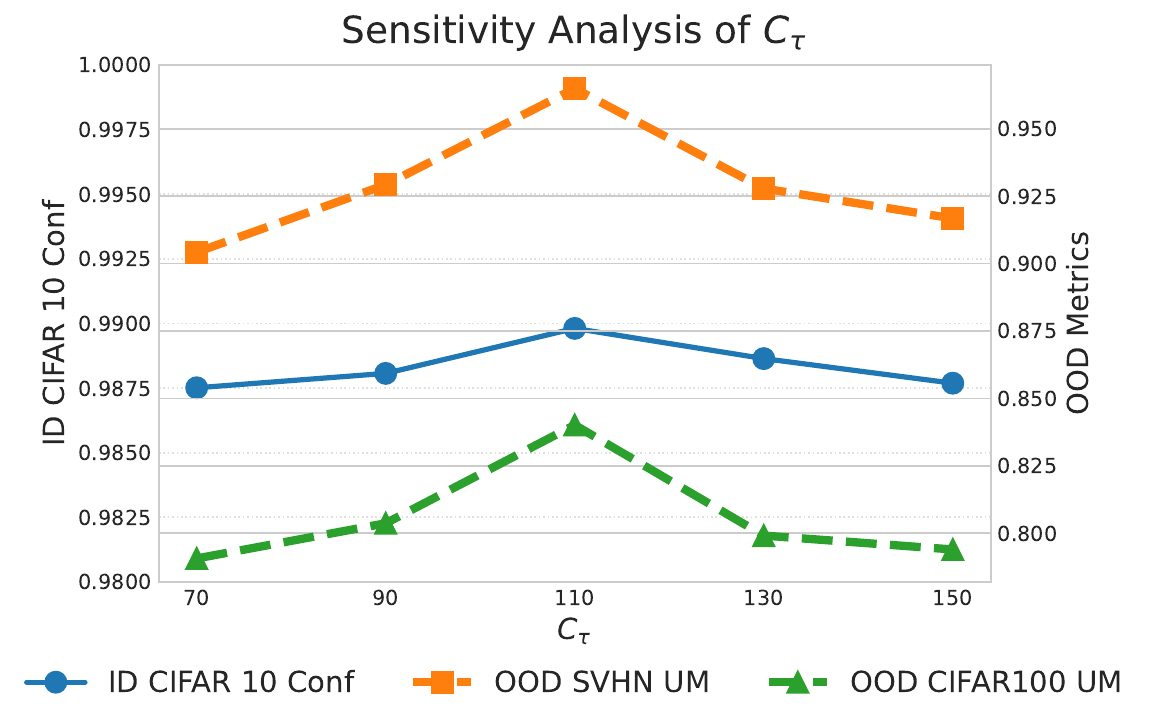}
        \caption{Sensitivity analysis w.r.t. evidence strength constant $C_{\tau}$.}
        \label{fig:ctau_sensitivity}
    \end{subfigure}
    \caption{Parameter sensitivity analysis of GEDL. 
    The performance is evaluated on CIFAR-10 (ID) and SVHN / CIFAR-100 (OOD) under different choices of hyperparameters.}
    \label{fig:param_analysis}
\end{figure}

\subsection{More baseline comparisons}

\begin{table}[h]
\centering
\caption{Extended comparison with additional baselines on MNIST and CIFAR-10. Results are reported as mean $\pm$ std over 5 runs. \textbf{Bold} indicates the best and \underline{underlined} indicates the second best. Methods highlighted with a \colorbox{gray!15}{gray background} are trained and evaluated during the rebuttal stage.}
\label{tab:additional_baselines}
\begin{adjustbox}{max width=\textwidth}
\begin{tabular}{l cc cc cc}
\toprule
& \multicolumn{2}{c}{ID Classification} & \multicolumn{2}{c}{OOD Detection (KMNIST/SVHN)} & \multicolumn{2}{c}{OOD Detection (FMNIST/CIFAR-100)} \\
\cmidrule(lr){2-3} \cmidrule(lr){4-5} \cmidrule(lr){6-7}
Method & Acc ($\uparrow$) & Conf.MP ($\uparrow$) & MP ($\uparrow$) & UM ($\uparrow$) & MP ($\uparrow$) & UM ($\uparrow$) \\
\midrule
\multicolumn{7}{c}{\textit{MNIST}} \\
\midrule
EDL        & $98.22_{\pm.31}$ & \bm{$99.99_{\pm.00}$} & $97.02_{\pm.76}$ & $96.31_{\pm2.03}$ & $98.11_{\pm.44}$ & $98.08_{\pm.42}$ \\
I-EDL      & $99.21_{\pm.08}$ & {$99.98_{\pm.00}$} & $98.34_{\pm.24}$ & $98.33_{\pm.24}$ & $98.89_{\pm.28}$ & $98.86_{\pm.29}$ \\
R-EDL      & $99.33_{\pm.03}$ & \bm{$99.99_{\pm.00}$} & $98.69_{\pm.19}$ & $98.69_{\pm.20}$ & \bm{$99.29_{\pm.11}$} & $99.29_{\pm.12}$ \\
\rowcolor{gray!15}
RED        & $99.38$ & $99.97 $& $98.71$ & $97.24$ & $99.34$ & $98.78$ \\
\rowcolor{gray!15}
DAEDL      & $99.43$ & \bm{$99.99$} & $98.92$ & \underline{$98.86$} & \underline{$99.40$} & \bm{$99.54$} \\
\rowcolor{gray!15}
F-EDL      & $99.48$ & \bm{$99.99$} & \underline{$99.02$} & $96.71$ & \bm{$99.66$} & $94.51$ \\
GEDL (Ours) & \bm{$99.58_{\pm.02}$} & $99.95_{\pm.01}$ & \bm{$99.80_{\pm.01}$} & \bm{$99.70_{\pm.04}$} & $98.54_{\pm.13}$ & $98.99_{\pm.05}$ \\
\midrule
\multicolumn{7}{c}{\textit{CIFAR-10}} \\
\midrule
EDL        & $83.55_{\pm.64}$ & $97.86_{\pm.17}$ & $78.87_{\pm3.50}$ & $79.12_{\pm3.69}$ & $84.30_{\pm.67}$ & $84.18_{\pm.74}$ \\
I-EDL      & $89.20_{\pm.32}$ & $98.72_{\pm.12}$ & $83.26_{\pm2.44}$ & $82.96_{\pm2.17}$ & $85.35_{\pm.69}$ & $84.84_{\pm.64}$ \\
R-EDL      & $90.09_{\pm.30}$ & $98.98_{\pm.05}$ & $85.00_{\pm1.22}$ & $85.00_{\pm1.22}$ & $87.72_{\pm.31}$ & \bm{$87.73_{\pm.31}$} \\
\rowcolor{gray!15}
RED        & 89.80 & 98.63 & 90.10 & \underline{$89.92$} & 84.97 & 84.62 \\
\rowcolor{gray!15}
DAEDL      & 90.06 & 98.80 & 85.28 & 85.30 & 84.34 & 84.50 \\
\rowcolor{gray!15}
F-EDL      & \bm{$91.29$} & \bm{$99.10$} & 88.57 & 61.04 & \underline{88.18} & 68.79 \\
GEDL (Ours) & \underline{$90.22_{\pm.49}$} & $97.15_{\pm3.72}$ & \bm{$89.26_{\pm3.26}$} & \bm{$92.56_{\pm6.02}$} & \bm{$88.27_{\pm2.51}$} & \underline{$85.42_{\pm5.14}$} \\
\bottomrule
\end{tabular}
\end{adjustbox}
\end{table}

\subsection{Calibration evaluations}

\begin{table}[h]
\centering
\caption{Calibration evaluation on CIFAR-10. ECE ($\downarrow$) and Brier Score ($\downarrow$) are reported. GEDL is mean $\pm$ std over 5 runs, others are single run. \textbf{Bold} indicates the best and \underline{underlined} indicates the second best. Methods highlighted with a \colorbox{gray!15}{gray background} are trained and evaluated during the rebuttal stage.}
\label{tab:calibration_cifar10}
\begin{adjustbox}{max width=\textwidth}
\begin{tabular}{l
>{\columncolor{gray!15}}c
>{\columncolor{gray!15}}c
>{\columncolor{gray!15}}c
>{\columncolor{gray!15}}c
>{\columncolor{gray!15}}c
>{\columncolor{gray!15}}c
c}
\toprule
Metric & EDL & I-EDL & R-EDL & RED & DAEDL & F-EDL & \textbf{GEDL (Ours)} \\
\midrule
ECE ($\downarrow$) 
& 0.1294 
& 0.3858 
& 0.0534 
& 0.0497 
& \underline{0.0443} 
& \bm{$0.0408$} 
& ${0.0464_{\pm0.0065}}$ \\

Brier ($\downarrow$) 
& 0.1965 
& 0.3145 
& 0.1820 
& 0.1720 
& 0.1622 
& \bm{$0.1348$} 
& \underline{$0.1525_{\pm0.0071}$} \\
\bottomrule
\end{tabular}
\end{adjustbox}
\end{table}

\end{document}